\documentclass{svjour3}
\usepackage{geometry}
\renewcommand{\qed}{\hfill\rule{1ex}{1ex}}
\usepackage{color}
\usepackage{amsfonts,amsmath,amssymb}
\usepackage{graphicx}
\usepackage{epsfig,tikz}

\usepackage{pgfplots}
\pgfplotsset{compat=1.11}

\newtheorem{prop}{Proposition}

\newcommand{\comment}[1]{}

\newcommand{\noauthor}[1]{\author{ Anonymous}}
\newcommand{\nothanks}[1]{}
\newcommand{\noauthorrunning}[1]{}
\newcommand{\noinstitute}[1]{}

\begin{document}

\title{Optimal Decision Trees for Categorical Data via Integer Programming
\thanks{
 The work of Katya Scheinberg was partially supported by NSF Grant CCF-1320137. Part of this work was performed while 
 Katya Scheinberg was on sabbatical leave at IBM Research, Google,  and  University of Oxford, partially supported by the Leverhulme Trust. }
}

\titlerunning{Optimal Decision Trees }        

\author{       Oktay G{\"u}nl{\"u}k \and Jayant Kalagnanam \and Minhan Li \and Matt Menickelly \and  Katya Scheinberg}

\authorrunning{     G{\"u}nl{\"u}k, Menickelly, Li,  Kalagnanam,    Scheinberg}

\institute{\\   \and    
            Oktay G{\"u}nl{\"u}k,    IBM  Research,   \email{gunluk@us.ibm.com}  \\\and
            Jayant Kalagnanam,   IBM  Research,   \email{jayant@us.ibm.com}       \\\and
            Minhan Li, Lehigh University, \email{mil417@lehigh.edu}\\\and
            Matt Menickelly,  Argonne National Laboratory,   \email{mmenickelly@anl.gov }    \\\and 
            Katya Scheinberg, Cornell University,   \email{katyas@cornell.edu}
}

\date{\today}

\maketitle

\begin{abstract}
Decision trees have been a very popular class of predictive models for decades due to their interpretability and good performance on categorical features. However, they are not always robust and tend to overfit the data. Additionally, if allowed to grow large, they lose interpretability. In this paper, we present a mixed integer programming formulation to construct optimal decision trees of a prespecified size. We take the special structure of categorical features into account and allow combinatorial decisions  (based on subsets of values of features) at each node. Our approach can also handle numerical features via thresholding. We show that very good accuracy can be achieved with small trees using moderately-sized training sets. The optimization problems we solve are tractable with modern solvers. 
\keywords{Decision Trees \and Integer Programming  \and Machine Learning \and Binary Classification}
\end{abstract}

\large
\section{Introduction}

\usetikzlibrary{arrows}
\definecolor{burgundy}{rgb}{0.5, 0.0, 0.13}
\tikzset{
	treenode/.style = {align=center,  very thick},
	bucket/.style = {treenode,  rounded corners, inner sep=.1cm, rectangle, black,  draw=black,    fill=gray!50, text width=1em, text height=1em},
	bucketB/.style = {treenode,  rounded corners, inner sep=.1cm, rectangle, black,  draw=black, text width=1em, text height=1em},
	decision/.style = {treenode, circle, burgundy, draw=burgundy},
	topstart/.style = {treenode,   minimum width=0.5em, minimum height=1.5em}
}
Interpretability has become a well-recognized goal for machine learning models as they push further into domains such as medicine, criminal justice, and business. 
In many of these applications machine learning models complement domain experts and for human decision-makers to trust these models, interpretability is crucial.
Decision trees have been a very popular class of predictive models for decades due to their interpretability and good performance on categorical features. 
Decision trees (DTs, for short) are similar to flow-charts as they  apply a sequence of binary tests or \emph{decisions} to predict the output label of the input data.
As they can be easily interpreted and applied by non-experts, DTs are considered as one of the most widely used tools of machine learning and data analysis 
(see the recent survey \cite{Kotsiantis2013} and references therein).
Another advantage  of DTs is that they often naturally result in feature selection, as only a part of the input is typically used in the decision-making process. 
Furthermore, DTs can work with both numerical and categorical data directly, which is not the case for  \emph{numerical classifiers} such as  linear classifiers or neural networks, 
as these methods require the data to be real-valued (and ordinal).
For example, if a categorical feature can take three values such as $(i)$ red,  $(ii)$ blue, or,  $(iii)$ yellow, 
it is often represented by a group of three binary features such that one of these features  takes the value $1$ while the other two are $0$.
A numerical classifier would treat this group of three features independently where any combination of 0/1 values are possible -  ignoring the valuable information 
that only three values for the triplet are possible. 
Numerical classifiers typically recover this lost information by observing enough data and fitting the model accordingly.
However, this is not a trivial task, and may require a more complex model than what is really necessary. 
In comparison, DTs can explicitly deal with categorical features.

There are also known disadvantages to DT predictors.
For example, they are not always robust, as they might result in poor prediction on out-of-sample data when the tree is grown too large.
Hence, small trees are often desirable to avoid overfitting and also for the sake of interpretability.
 Assuming that for a given data distribution there exists a small DT that can achieve good accuracy,
the small DTs that are computed by a typical recursive DT algorithm (such as CART \cite{Breimanbook,Quinlan})  may not achieve
such accuracy, due to the  heuristic nature of the algorithm. 
Moreover, it is usually impossible to establish a bound on the difference between the expected accuracy of the DT produced by a heuristic algorithm and the best possible DT.

Currently, popular algorithms used for constructing DTs (such as CART or C4.5) are sequential heuristics that first construct a tree and then trim (prune) it to reduce its size, 
see \cite{Kotsiantis2013}.
When building the tree, these heuristics use various criteria to choose a feature and a condition on that feature to branch on. 
As the tree is built gradually, the resulting DT is not necessarily ``the best'' for any particular global criterion.
One recent example of this fact is the winning entry  \cite{fico_oktay} in the FICO interpretable machine learning competition  \cite{FICO}. The authors of \cite{fico_oktay} construct a simple classifier in conjunctive normal form which in fact can also be seen as a small depth decision tree. The authors show that their classifier is both simpler and more accurate (on test data) than the trees constructed by CART.

In this paper, we aim to find {\em optimal} small DTs for binary classification problems 
that produce \emph{interpretable} and \emph{accurate} classifiers for the data for which such classifiers exist.
We call a DT optimal if it has the best possible classification accuracy on a given training dataset.
 We allow complex branching rules using subsets of values of categorical features. 
 For example, if a categorical feature represents a person's marital status
and can take the values ``single", ``married",``divorced", ``widowed", or ``has domestic partner", 
a simple branching rule, which looks at numerical representation of the features, will make decisions based on a feature being  ``single" or not,
while a more appropriate decision may be ``either married or has a domestic partner" or not. 
Such combinatorial branching rules are considered desirable and  in the case of binary classification using CART, branching on the  best subset values of a categorical feature can be done again according to a sequential local heuristic.
On the other hand, combinatorial branching may lead to overfitting when a categorical variable can take a large number of values.
If the categorical variable that can take $\ell$ values, then, there are $2^\ell-2$ possible subsets of values of this feature that can be used for branching. 
To avoid overfitting,  our model allows bounding the size of the subset used for branching.


 While finding an optimal DT (even without the combinatorial decisions) is known to be an NP-hard problem \cite{hyafil_rivest}, we show that with careful modeling,
the resulting integer programs can be solved to optimality in a reasonable amount of time using commercial solvers such as Cplex.
Moreover, since we directly optimize the empirical loss of a DT in our model, even suboptimal feasible solutions tend to yield classifiers that outperform those learned by other DT algorithms. 
In particular, we consider a binary classification problem, which means that the output nodes (leaves) of our DTs
generate binary output. 
Our problem formulation takes particular advantage of this fact. 
Also, while our formulation can be generalized to real-valued data, it is designed for the case when the input data is binary. 
Hence, we will consider input data as being a binary vector with the property that features are grouped so that only one feature can take the value $1$ in each group for each data sample.
Our formulation explicitly takes this structure into account as we allow branching on any subset of the values of that feature. 
To our knowledge such generalized rules have not been addressed by any algorithm aiming at constructing of optimal trees, such as a recent 
method proposed in \cite{bertsimas_dunn}, which we will discuss in the next section. 

In this paper, we focus on constructing small DTs with up to four levels of decisions, which makes the resulting model
clearly interpretable and easily usable by humans. 
Our formulation, in principle, can work for binary trees of any topology; however, as we will show in our computational results,
 trees of more complex topologies are much more time consuming to train and require larger training sets to avoid overfitting. 
 The purpose of this paper is to show that if an accurate small (interpretable) tree exists for a given data set, it can be obtained in a reasonable 
 time by our proposed model, while popular heuristic methods such as C4.5 \cite{Quinlan} and random forests \cite{BreimanRF}
 tend to produce less accurate and less interpretable trees. 
 We note that even though we mostly focus on categorical features, our approach can easily handle numerical features via tresholding. We discuss how to do this this later  and also present numerical experiments with data sets with both categorical and numerical features.

The key approach we pursue is to formulate the DT training problem as a mixed-integer optimization problem that is specially designed to handle categorical variables. 
We then propose several modifications that are intended to aid a branch and bound solver, e.g. symmetry breaking. 
We also consider an extension to a formulation
that directly constrains either training sensitivity or training specificity and then maximizes the other measure. 

The rest of the paper is organized as follows: 
First, in Section~\ref{sec:related_work}, we discuss related work in using integer formulations for practical machine learning.
Then, in Section~\ref{sec:setting}, we describe the main ideas of our approach and the structure of the data for which
the model is developed. 
In Section \ref{sec:IP_section} we describe an initial IP model and several techniques
for strengthening this formulation. 
We present some computational results and comparisons in Section \ref{sec:comp}. 

\section{Related Work}\label{sec:related_work}
The idea of solving decision trees to optimality given a fixed topology is hardly new.
In \cite{Breimanbook} from 1984, the authors discuss the ``one-step optimality'' of inductive (greedy) tree algorithms,
and how one would ideally prefer an ``overall optimal'' method wherein the tree is learned in one step (such as the one we explore in this paper). 
The authors remark that this is analogous to a ``best subset selection'' procedure of linear regression, 
and continue to say that ``At the current stage of computer technology, an overall optimal tree growing procedure does not appear feasible for any
reasonably sized dataset''. 
In \cite{murthy_salzberg}, the authors detail what they call the ``look-ahead pathology'' of greedy tree learning algorithms, lending further evidence
of possible failures of greedy one-step methods. 

In the 1990s several papers considered optimization formulations for optimal decision tree learning, but deliberately relaxed the inherently integer nature of the problem. 
In particular, in \cite{bennett_blue1}, a large-scale linear optimization problem, which can be viewed as a relaxation, is solved to global optimality
via a specialized tabu search method over the extreme points of the linear polytope. 
In \cite{bennett_blue2}, a similar formulation is used, but this time combined with the  use of support-vector machine techniques such as generalized kernels for multivariate decisions,
yielding a convex nonlinear optimization problem which admits a favorable dual structure. 
More recent work \cite{norouzi_et_al} has employed a stochastic gradient method to minimize a continuous upper bound on misclassification error made by a deep decision tree. 
None of these methods, though, guarantee optimal decision trees, since they do not consider the exact (integer) formulations, such as the one discussed in this paper. 

Recently, in  \cite{bertsimas_dunn}, an integer model for optimal decision trees has been proposed. 
The key difference with the model in this paper is that \cite{bertsimas_dunn} does not target categorical variables and, hence, does not exploit the resulting combinatorial structure. Moreover, all features are treated as real-valued ones, hence a categorical feature is replaced by several binary features, and two possible models are proposed. 
The first uses arbitrary linear combinations of features, and, in principal, is more general than what we propose here, but results in a loss of interpretability.
The second uses the value of one feature in each branching decision, and hence is less general than the model in this paper. 
Additionally, we focus on binary classification problems whereas \cite{bertsimas_dunn} presents a formulation for  multi-class classification. 
Rather than fixing a tree topology, as we do, they propose tuning a regularization parameter in the objective;
as the parameter magnitude increases, more leaf nodes may have no samples routed to them, effectively yielding shallower trees. 
We note that this does not simplify the underlying optimization problem, and moreover requires tuning parameters in a setting where the training of models is computationally non-negligible,
and the effect of the choice of regularization parameter on the tree topology cannot be known a priori. In fact, in the computational results of \cite{bertsimas_dunn}, the depth is often fixed. 
Finally, unlike the work in \cite{bertsimas_dunn}, we not only propose a basic model that specifically exploits the categorical nature of the features, but we also propose several 
modifications of the model that produce stronger formulations and improve the efficiency of the branch and bound solver. 

We would now like to remark on other relevant uses of integer optimization in classification settings. 
In particular, \cite{wang_rudin} considered the problem of learning optimal ``or's of and's'', which fits into the problem of learning optimal 
disjunctive normal forms (DNFs), where optimality is measured by a trade-off between the misclassification rate and the number of literals that appear in the ``or of and's''.
The work in \cite{wang_rudin} remarks on the relationship between this problem and learning optimal decision trees.
In \cite{wang_rudin},
for the sake of computational efficiency, the authors ultimately resort to
optimally selecting from a subset of candidate suboptimal DNFs learned by  heuristic means rather than solving their proposed mixed-integer optimization problem. 
Similarly, \cite{malioutov_varshney} proposes learning DNF-like rules via integer optimization, and propose a formulation that can be viewed as boolean compressed sensing, lending  theoretical credibility to solving a linear programming relaxation of their integer problem. 
Another integer model that minimizes misclassification error by choosing general partitions in feature space was proposed in \cite{bertsimas_shioda}, but when solving the model,
global optimality certificates were not easily obtained on moderately-sized classification datasets, and the learned partition classifiers rarely outperformed CART, according
to the overlapping author in \cite{bertsimas_dunn}.
Finally,  a column generation based mixed-integer programming approach to construct optimal DNFs was recently proposed in  \cite{fico_oktay}. 
	This approach seems to work quite well on several binary classification datasets including the FICO challenge data \cite{FICO}. 

\section{Setting}\label{sec:setting}

In this paper we  consider datasets of the form $\{(g^i_1,\ldots,g^i_t,y^i):i\in 1,2,\dots, N\}$ where $g^i_j\in G_j$ for some finite set $G_j$ for $j=1,\ldots,t$, and $y^i\in\{-1,+1\}$ is the class label associated with a negative or positive class, respectively. 
For example, if the data is associated with a  manufacturing process with $t$ steps, then each  $G_j$ may correspond to a collection of different tools that can perform the $j$th step of the production process and the label may denote whether  the resulting product meets certain quality standards or not.
The classification problem associated with such an example is to estimate the label of a new item based on the particular different step-tool choices used in its manufacturing.
Alternatively, the classification problem can involve estimating whether a student will succeed in graduating from high school based on features involving gender, race,  parents marital status, zip-code and similar information.

Any (categorical) data of this form can alternatively be represented by a binary vector so that  $g^i_j\in G_j$ is replaced by a unit vector of size $|G_j|$  where the only non-zero entry in this vector indicates the particular member of $G_j$ that the data item contains.
In addition, a real-valued (numerical) feature can be, when appropriate, made into a categorical one by  ``bucketing" - that is breaking up the range of the feature into segments and considering segment membership as a categorical feature.
This is commonly done with features such as  income or age of an individual.
For example, for advertising purposes websites typically represent users by age groups such as ``teens", ``young adults", ``middle aged", and ``seniors" instead of actual age.

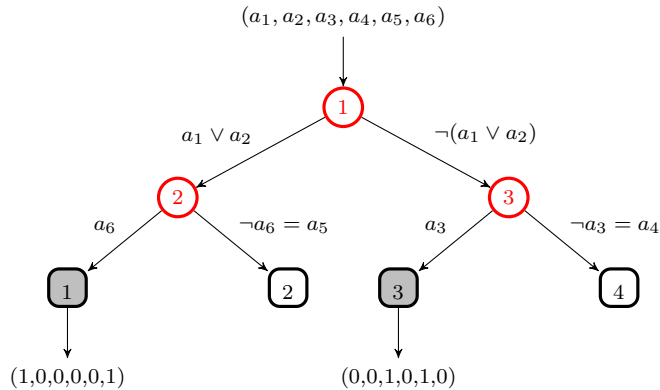
\begin{figure} \label{fig:dectree} \caption{A decision tree example} 
\centering
\begin{tikzpicture}[scale=0.8,->,>=stealth',level/.style={sibling distance = 11cm/#1,  level distance = 1.5cm}] 
\node [topstart] {$(a_1,a_2,a_3,a_4,a_5,a_6)$}
child{node [decision] {1} 
	child{ node [decision] {2} 
		child{ node [bucket] {1}       child{ node [topstart] {(1,0,0,0,0,1)} }  	edge from parent node[above left]  {$a_6$}           }
		child{ node [bucketB] {2}          edge from parent node[above right]  {$\neg a_6=a_5$}        }                            
		edge from parent node[above left]  {$a_1\vee a_2$}    
	}
	child{ node [decision] {3}
		child{ node [bucket] {3}       child{ node [topstart] {(0,0,1,0,1,0)} } 	edge from parent node[above left]  {$a_3$}            }
		child{ node [bucketB] {4}       	edge from parent node[above right]  {$\neg a_3=a_4$}           }
		edge from parent node[above right]  {$\neg (a_1\vee a_2)$}   	
	}
}
; 
\end{tikzpicture}
\end{figure}
The non-leaf nodes in a decision tree are called the \emph{decision nodes} where a binary test is applied to data items.
Depending on the results of these tests, the data item is routed to one of the {\em leaf} nodes.
Each leaf node is given a binary label  that determines the label assigned to the data by the DT.
The binary tests we consider are of the form ``does the $j$th feature of the data item belong to set  $\bar G_j$?'', where $\bar G_j\subseteq G_j $.
If the categorical data is represented by a binary vector, then the test becomes checking if at least one of the indices from a given collection contains a 1 or not.

As a concrete example, consider the tree in Figure \ref{fig:dectree} applied to binary vectors $a\in\{0,1\}^6$ 
whose elements are divided into  two groups: $\{a_1, a_2, a_3, a_4\}$ and $\{a_5,  a_6\}$ corresponding to two categorical features in the original data representation. 
The branching decision at node $1$ (the root), is based on whether one of $a_1$ or $a_2$ is equal to $1$. 
If true, a given data sample is routed to the left, otherwise (that is, if both $a_1$ and $a_2$ are $0$), the sample is routed to the right. 
The branching at nodes $2$ and $3$ (the two children of node $1$) are analogous and are shown in the picture. 
We can now see that data sample $a^1=(1, 0, 0, 0, 0, 1)$ is routed to leaf node $1$  and data sample $a^2=(0,0,1, 0, 1,0)$ is routed to leaf node $3$. 
The labels of the leaf nodes are denoted by the colors white and gray in  Figure \ref{fig:dectree}.

Formally, a DT is defined by $(i)$ the topology of the tree, $(ii)$   binary tests applied at each decision node, and, $(iii)$  labels assigned to each leaf node.
Throughout the paper we consider tree topologies where a decision node either has two leaf nodes or else has two other decision nodes as children.
Note that  decision trees defined this way are inherently symmetric objects, in the sense that the same DT can be produced by different numberings of the decision and leaf nodes as well as different labeling of the leaf nodes and the binary tests applied at the decision nodes. For example, reversing the  binary test from $(a_6)$ to $(\neg a_6)$ in decision node 2, and at the same time  flipping the labels of the leaf nodes 1 and 2, results in an identical DT. More generally, it is possible to  reverse the  binary test at any decision node and ``flip" the  subtrees rooted at that node to obtain the same tree.

The optimization problem we consider in the next section starts with a given tree topology and finds the best binary tests (and labels for the leaf nodes) to classify  the test data at hand with minimum error. 
Due to the symmetry discussed above, we can fix the labeling of the leaf nodes at the beginning of the process and the problem reduces to finding the best binary tests, or equivalently, choosing a categorical feature and a subset of its realizations at each decision node.
Therefore, the optimization problem consists of assigning a binary test to each decision node so as to maximize the number of {correctly classified}  samples in the \emph{training set}. We say that 
the classification of the $i$th sample is \emph{correct}  provided the path the $i$th sample takes through the tree starting from the root node ends at a leaf corresponding to the correct label.
The ultimate goal of the process, however, is to obtain a DT that will classify new data well, i.e., we are actually concerned with the generalization ability of the resulting DT.

Notice that given two tree topologies such that one is a minor of the other (i.e. it can be obtained from the other by deleting  nodes and contracting edges),  the classification error of the best DT built using the larger tree will always be at least as small as that of the smaller tree. 
Consequently, for optimization purposes, larger trees  always perform better than any of its minors.
However, larger trees generally result in more computationally challenging optimization problems.
In addition, smaller trees are often more desirable for classification purposes as they are more robust 
and are easier to interpret.


\section{Integer Programming Formulation}
\label{sec:IP_section}
In this section, we first present the basic integer programming formulation 
and then describe some enhancements  to improve its computational efficiency.
We initially assume that the topology of the binary tree is given (see Figure~\ref{fig:tree}) and therefore the number of decision and leaf nodes as well as how these nodes are connected is  known. 
We will then describe how to pick a good topology.
The formulation below models how the partitioning of the samples is done at the decision nodes, and which leaf node each sample is routed to as a result.

We begin by introducing the notation.
Let the set of all samples be indexed by $I = \{1,2,\dots,N\}$, let $I_{+}\subset I$
denote the indices of samples with positive labels 
and let $I_{-}=I\setminus I_{+}$ denote the indices of the samples with negative labels.
Henceforth, we assume that for each sample the input data is transformed into a binary vector where each categorical feature is represented by a unit vector that indicates the realization of the categorical feature.
With some abuse of terminology,  we will now refer to the entries of this binary vector as ``features'', and the collection of these 0/1 features that are associated with the same categorical feature as ``groups''.
Let the set of groups be indexed by $G=\{1,2,\dots,|G|\}$ and the set of the 0/1 features be indexed by $ J =\{1,2,\dots,d\}$.
In addition, let  $g(j)\in G$ denote  the group that contains feature $j\in J$ and let $J(g)$ denote the set of features that are contained in group $g$.
Let the set of decision nodes 
be indexed by $K=\{1,2,\dots,|K|\}$ and the set of leaf nodes be indexed by $B=\{1,2,\dots,|B|\}$.
We denote the indices of leaf nodes with positive labels by $B_{+}\subset B$ and the indices of
leaf nodes with negative labels by $B_{-} = B\setminus B_{+}$.
For convenience,
we let $B_{+}$ contain even indices, and $B_{-}$ contain the odd ones.

\subsection{The basic formulation}\label{sec:basic}

We now describe our key decision variables and the constraints on these variables.
We use binary variables $v_g^k\in\{0,1\}$ for  $g\in G$ and $ k\in K$ to denote if group $g$ is selected for branching at decision node $k$.
As discussed in Section \ref{sec:setting}, exactly one group has to be selected for branching at a decision node; consequently, we have the following set of constraints:
\begin{equation}
\label{eq:one-group-per-node}
\displaystyle\sum_{g\in G} v_g^k = 1 \quad \forall k\in K.
\end{equation}

The second set of  binary variables $ z_j^k\in\{0,1\}$ for  $j\in J$ and $ k\in K$ are used to denote if feature $j$ is one of the selected features  for branching at a decision node $k$. 
Clearly,  feature $j\in J$ can be selected only if the group containing it, namely  $g(j),$ is selected at that node.
Therefore, for all  $ k\in K$ we have the following set of constraints:
\begin{equation}
\label{eq:decision_hierarchy}
z_j^k \leq v^k_{g(j)} \quad \forall j\in J,
\end{equation}
in the formulation.
Without loss of generality, we use the convention that if a sample has one of the selected features at a given node, it follows the left branch at that node; otherwise it follows the right branch.

Let
\begin{eqnarray*}
	S&=&\Big\{(v,z)\in\{0,1\}^{|K|\times|G|}\times\{0,1\}^{|K|\times d}:
	~~~(v,z)  \text{ satisfies inequalities \eqref{eq:one-group-per-node} and \eqref{eq:decision_hierarchy}}\bigg\},
\end{eqnarray*}
and note that for any $(v,z)\in S$ one can construct a corresponding decision tree in a unique way and vice versa.
In other words, for any given $(v,z)\in S$ one can easily decide which leaf node each sample is routed to.
We next describe how to relate these variables (and therefore the corresponding decision tree) to the samples.

We use binary variables $c_b^i\in\{0,1\}$ for $b\in B$ and $i\in I$ to denote if sample $i$ is routed to leaf node $b$.
This means that  variable $c_b^i$  should take the value 1 only when sample $i$ exactly follows the unique path in the decision tree that leads to  leaf node $b$.
With this in mind, we define the expression
\begin{equation}
\label{l_definition}
L(i,k) = \displaystyle\sum_{j\in J} a_j^iz_j^k\quad \forall k\in K,~\forall i\in I,
\end{equation}
and make the following observation:

\begin{prop}\label{prop1}
	Let $(z,v)\in S$. Then, for all $i\in I$ and $k\in K$ we have $L(i,k)\in\{0,1\}$ . 
	Furthermore, $L(i,k)=1$ if and only if there exists some $j\in J$ such that  $a_{j}^i=1$ and $z_{j}^k=1$.
\end{prop}
\begin{proof}
	For any $(z,v)\in S$ and $k\in K$, exactly one of the $v_g^k$ variables, say $v_{g'}^k$, takes value $1$ and $v_g^k=0$ for all $g\not= g'$.
	Therefore, $z_{j}^k=0$ for all $j\not\in J(g).$
	Consequently, the first part of the claim  follows for all $i\in I$ as   $L(i,k) = \sum_{j\in J} a_j^iz_j^k=\sum_{j\in J(g')} a_j^iz_j^k= z_{j_i}^k\in\{0,1\}$ where $j_i\in J(g')$ is the index of the unique feature for which $a_{j_i}^i=1$.
	In addition, $L(i,k)=1$ if and only if $z_{j_i}^k=1$ which proves the second part of the claim.\qed
\end{proof}

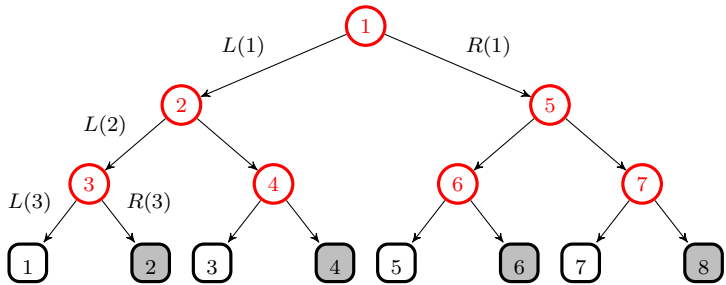
\begin{figure} [h] \caption{A balanced depth 3 tree} \centering \begin{tikzpicture}[scale=0.7,->,>=stealth',level/.style={sibling distance = 7cm/#1,  level distance = 1.5cm}] 
	\node [decision] {1}
	child{ node [decision] {2} 
		child{ node [decision] {3} 
			child{ node [bucketB] {1} edge from parent node[above left]  {$L(3)$}} 
			child{ node [bucket] {2} edge from parent node[above right]  {$R(3)$}} 
			edge from parent node[above left]  {$L(2)$} 
		}
		child{ node [decision] {4}
			child{ node [bucketB] {3}}
			child{ node [bucket] {4}}
			edge from parent node[above right]  {$R(2)$} 
		}         
		edge from parent node[above left]  {$L(1)$} 
	}
	child{ node [decision] {5}
		child{ node [decision] {6} 
			child{ node [bucketB] {5}}
			child{ node [bucket] {6}}
		}
		child{ node [decision] {7}
			child{ node [bucketB] {7}}
			child{ node [bucket] {8}}
		}
		edge from parent node[above right]  {$R(1)$} 
	}
	; 
	\end{tikzpicture}
	\label{fig:tree}\end{figure}

Consequently, the expression $L(i,k)$ indicates if sample $i\in I$ branches left at node $k\in K$. Similarly, we define the expression
\begin{equation}
\label{r_definition}
R(i,k)=1-L(i,k)\quad \forall k\in K,~\forall i\in I,
\end{equation}
to indicate if sample $i$ branches right at node $k$.

To complete the model, we relate the expressions $L(i,k)$ and $R(i,k)$  to the $c_b^i$ variables.
Given that the topology of the tree is fixed, there is a unique path leading to each leaf node $b\in B$ 
from the root of the tree. This path visits a subset of the nodes $K(b)\subset K$ and for each $k\in K(b)$ 
either the left branch or the right branch is followed.
Let $K^L(b)\subseteq K(b)$ denote the decision nodes where the left branch is followed to reach
leaf node $b$ and let $K^R(b)=K(b)\setminus K^L(b)$ denote the decision nodes where the right branch is followed.
Sample $i$ is routed to $b$ only if it satisfies all the conditions at the nodes leading to that leaf node.
Consequently, we define the constraints
\begin{eqnarray}
\label{eq:left} c^i_b\le L(i,k)~~\text{ for all } \quad \forall b\in B,~\forall i\in I,~k\in K^L(b),\\
\label{eq:right} c^i_b\le R(i,k)~~\text{ for all } \quad \forall b\in B,~\forall i\in I,~k\in K^R(b),
\end{eqnarray}
for all $i\in I$ and $b\in B$.
Combining these  with the equations
\begin{equation}
\label{eq:pickone}
\sum_{b\in B}c^i_b=1 \quad \forall i\in I
\end{equation}
gives a complete formulation. 
Let
$$Q(z,v) =\big\{c\in\{0,1\}^{N\times|B|}:\text{ such that \eqref{eq:left}-\eqref{eq:pickone} hold}\big\}.$$
We next formally show that combining the constraints in $S$ and $Q(z,v)$ gives a correct formulation. 
\begin{prop}\label{prop2}
	Let $(z,v)\in S$, and let $c\in Q(z,v)$. Then, $c^i_b\in\{0,1\}$ for all $i\in I$ and $b\in B$.
	Furthermore, if $c^i_b=1$ for some $i\in I$ and $b\in B$, then sample $i$ is routed to leaf node $b$.
\end{prop}
\begin{proof} 
	Given $(z,v)\in S$ and $i\in I$, assume that the correct  leaf node sample $i$ should be routed to in the  decision tree defined by $(z,v)$ is the leaf node $b'$. 
	For all other leaf nodes $b\in B\setminus\{b'\}$,  sample $i$ either has $L(i,k)=0$ for some $ k\in K^L(b)$ or $R(i,k)=0$ for some $ k\in K^R(b)$. Consequently, $c^i_b=0$ for all $b\not=b'$. Equation \eqref{eq:pickone} then implies that  $c^i_{b'}=1$ and therefore $c^i_b\in\{0,1\}$ for all $b\in B$.
	Conversely, if  $c^i_{b'}=1$ for some $b'\in B$, then $L(i,k)=1$ for all $ k\in K^L(b)$ and $R(i,k)=1$ for all $ k\in K^R(b)$.\qed
\end{proof}

We therefore have the following integer programming (IP) formulation:
\begin{subequations}\label{firstmodel}\begin{align}
	\max & \hskip.5cm\sum_{i\in I_{+}}\sum_{b\in B_{+}} c_b^i + C\sum_{i\in I_{-}}\sum_{b\in B_{-}} c_b^i\label{firstmodel:obj}\\[.3cm]
	\text{s. t.} &\hskip1cm (z,v)\in S\\[.2cm] &\hskip1cmc\in Q(z,v)
	\end{align}\end{subequations}

\noindent where $C$ in the objective \eqref{firstmodel:obj} is a constant weight chosen in case of class imbalance. 
For instance, if a training set has
twice as many good examples as bad examples, it may be worth considering setting $C=2$, 
so that every correct classification
of a bad data point is equal to two correct classifications of good data points.

{
	Notice that formulation \eqref{firstmodel} allows solutions where all samples follow the same branch.
	For example,  it is possible to have a solution where a branching variable $ v^k_{g}=1$  for  some $k\in K$ and  $g\in G$, and at the same time  $z_j^k=0$ for all $j\in J(g)$. In this case $L(i,k)=0$ for all $i\in I$ and all samples follow the right branch.
	It is possible to exclude such solutions using the following constraint:
	
	\begin{equation}
	\label{eq:pickmore}\sum_{j\in J(g)} z_j^k \geq v^k_{g},
	\end{equation}
	where $J(g)\subset J$ denotes the set of features in group $g$.
	This constraint enforces that if a group is selected for branching at a node, then at least one of its  features has to be selected as well.
	Similarly, one can disallow the case when all samples follow the left branch:
	\begin{equation}
	\label{eq:pickless}\sum_{j\in J(g)} z_j^k \leq (|J(g)|-1)v^k_{g}.
	\end{equation}
	However, we do not use inequalities \eqref{eq:pickmore}-\eqref{eq:pickless} in our formulation and allow the decision tree nodes to branch in such a way that all items follow  the same branch.
}

\subsection{Choosing the tree topology}\label{sec:treetop}
\begin{figure} [tb] \caption{Possible tree topologies} \centering
	\begin{tikzpicture}[scale=0.6,->,>=stealth',level/.style={sibling distance = 7cm/#1,  level distance = 1.5cm}] 
	\node [decision] {1}
	child{ node [decision] {2} 
		child{ node [bucketB] {1}     }
		child{ node [bucket] {2}     }                        
	}
	child{ node [decision] {3}
		child{ node [bucketB] {3}             }
		child{ node [bucket] {4}            }
	}
	; 
	\end{tikzpicture}
	~~~~~~~\begin{tikzpicture}[scale=0.6,->,>=stealth',level/.style={sibling distance = 8cm/#1,  level distance = 1.5cm}] 
	
	\node [decision] {1}
	child{ node [decision] {2} 
		child{ node [decision] {3} 
			child{ node [bucketB] {1} }
			child{ node [bucket] {2}}
		}
		child{ node [decision] {4}
			child{ node [bucketB] {3}}
			child{ node [bucket] {4}}
		}                            
	}
	child{ node [decision] {5}
		child{ node [bucketB] {5}             }
		child{ node [bucket] {6}            }
	}
	
	; 
	\end{tikzpicture}

		\vskip1cm

\begin{tikzpicture}[scale=0.6,->,>=stealth',level/.style={sibling distance = 13cm/#1,  level distance = 1.5cm}] 
\node [decision] {1}
child{ node [decision] {2} 
	child{ node [decision] {3} 
		child{ node [decision] {4} child{ node [bucketB] {1}}		child{ node [bucket] {2}}}
		child{ node [decision] {5} child{ node [bucketB] {3}}		child{ node [bucket] {4}}} 
	}
	child{ node [decision] {6}
		child{ node [bucketB] {5}}
		child{ node [bucket] {6}}
	}         
}
child{ node [decision] {7}
	child{ node [bucketB] {7} }
	child{ node [bucket] {8}}
}

; 
\end{tikzpicture}
\label{fig:trees}\end{figure}
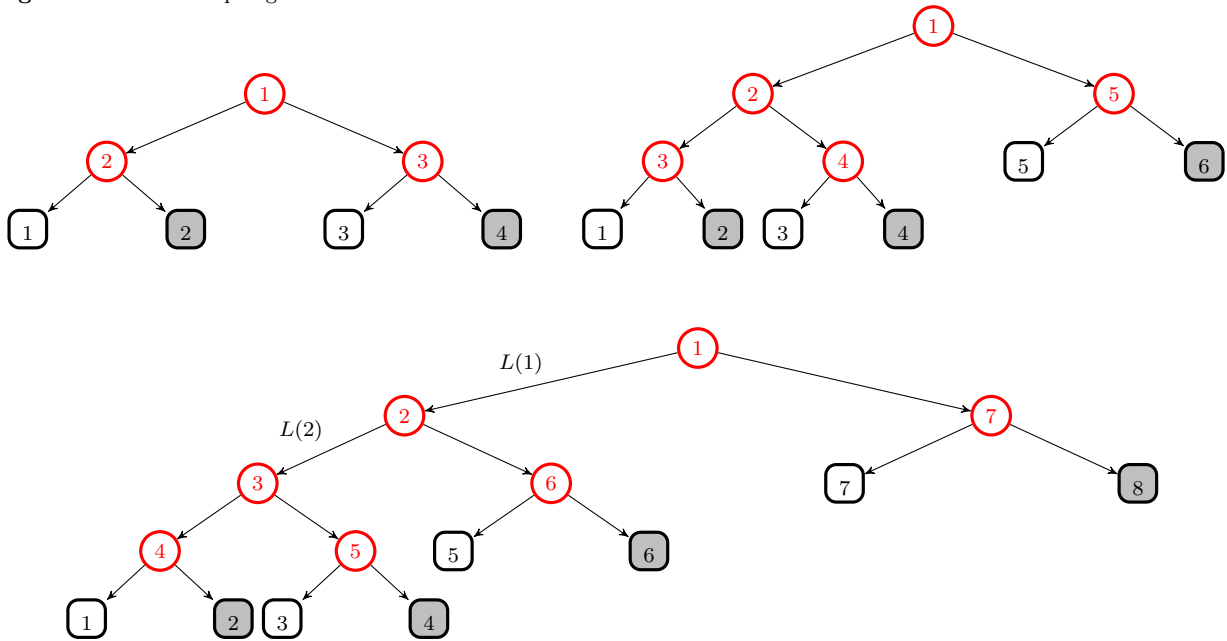
The IP model  \eqref{firstmodel} finds the optimal decision tree for a given tree topology which is an input to the model.
It is possible to build a more complicated IP model that can also build the tree topology (within some restricted class) but for computational efficiency, we decided against it.
Instead, for a given dataset, we use several fixed candidate topologies and build a different DTs for each one of them.
We then pick the most promising one using cross-validation.
The 4 tree topologies we use are the balanced depth 3 tree shown in  Figure~\ref{fig:tree} and the additional trees shown in Figure~\ref{fig:trees}. 

Note that the first two trees presented in Figure~\ref{fig:trees} can be obtained as a minor of the balanced depth 3 tree shown in  Figure~\ref{fig:tree} and therefore, the optimal value of the model using the balanced depth 3 tree will be better than that of using either one of these two smaller trees. 
Similarly, these two trees can also be  obtained as a subtree of the last tree  in Figure~\ref{fig:trees}.
However, due to possible overfitting, the larger trees might perform worse than the smaller ones on new data. As we will show via computational experiments, training smaller trees take fraction of the time compared to training larger trees, hence training a collections of trees of increasing topologies is comparable to training one large tree.

\subsection{ Computational tractability}\label{sec:tricks}
While \eqref{firstmodel} is a correct formulation, it can be improved to enhance computational performance.  
We next discuss some ideas that help reduce the size of the problem, break symmetry and strengthen the linear programming relaxation.
We first observe that the LP relaxation of \eqref{firstmodel}, presented explicitly below, is rather weak.

\begin{equation*}\begin{alignedat}{3}
\max 		 \hskip.7cm\sum_{i\in I_{+}}\sum_{b\in B_{+}} c_b^i &+ C\sum_{i\in I_{-}}\sum_{b\in B_{-}} c_b^i\label{LPfirstmodel:obj}\\[.3cm]
\text{s. t.} 	\hskip1.5cm \sum_{g\in G} v_g^k &= 1 \quad \forall k\in K,\\
z_j^k &\leq v^k_{g(j)} \quad \forall j\in J, ~\forall k\in K\\[.1cm] 
\sum_{b\in B: K^L(b) \ni k }c^i_b&\le L(i,k)\quad \forall i\in I, ~k\in K\\
\sum_{b\in B: K^R(b) \ni k }c^i_b&\le R(i,k)\quad \forall i\in I, ~k\in K.
\end{alignedat}
\end{equation*}
As $\sum_{b\in B} c_b^i \le 1$, for all $i\in I$, the optimal value of the LP relaxation is at most $| I_{+}|+C|  I_{-}|$. 
Assuming that  the decision tree has at least two levels, 
we will next construct a solution to the LP that attains this bound. 
Moreover, this solution would also satisfy $v^k_{g}\in\{0,1\}1$ for all $k\in K$ and $g\in G$. 

As the decision tree has at least two levels, both the left and right branches of the root node contain a leaf node in $ B_{+}$ as well as a leaf node in $ B_{-}$.
Let  $b_-^L,b_-^R\in B_-$ and  $b_+^L,b_+^R\in B_+$  where  $b_-^L$ and $b_+^L$ belong to the left branch and $b_-^R$ and $b_+^R$ belong to the right branch.
For an arbitrary $\bar g\in G$, 
we construct the solution $(z,v,c)$ as follows:
First we set $v^k_{\bar g}=1$ for all $k\in K$ and $ z_j^k=1/2$ for all $k\in K$ and $j\in J(\bar g)$. 
We then set $c_b^i=1/2$ for $b\in\{b_+^L,b_+^R\}$ for all $i\in I_{+}$ and  set $c_b^i=1/2$ for $b\in\{b_-^L,b_-^R\}$ for all $i\in I_{-}$.
We set all the remaining variables to zero.
Notice that $\sum_{b\in B_{-}} c_b^i = 1$ for $i\in I_{-}$ and $\sum_{b\in B_{+}} c_b^i = 1$ for $i\in I_{+}$ and therefore the value of this solution is indeed $| I_{+}|+C|  I_{-}|$.  
To see that the this solution is feasible, first note that  $ \sum_{g\in G} v^k_g = 1 $ for all $k\in K$ and $ z_j^k \leq v^k_{g(j)} $ for all $j\in J$ and $k\in K$. 
Also notice that $L(i,k)=R(i,k)=1/2$ for all $ i\in I$ and $k\in K$, which implies that \eqref{eq:strongleft} and \eqref{eq:strongright} are also satisfied for all $ i\in I$ and $k\in K$.

\subsubsection{Strengthening the model}\label{sec:strengthening}
Consider  inequalities \eqref{eq:left}
\begin{equation*}
c^i_b\le L(i,k)~~
\end{equation*}
for  $i\in I$, $b\in B$ and $k\in K^L(b)$  where $K^L(b)$ denotes the decision nodes where the left branch is followed to reach the leaf node $b$.
Also remember that $\sum_{b\in B}c^i_b=1$  for $i\in I$ due to equation \eqref{eq:pickone}.

Now consider a fixed $i\in I$ and $k\in K$.
If $ L(i,k)=0$, then  $c^i_b=0$ for all $b$ such that   $k\in K^L(b)$. 
On the other hand,  if $ L(i,k)=1$ then at most one $c^i_b=1$ for  $b$ such that   $k\in K^L(b)$. 
Therefore,
\begin{equation} \label{eq:strongleft}
\sum_{b\in B: K^L(b) \ni k }c^i_b\le L(i,k)
\end{equation}
is a valid inequality for all ${i\in I}$ and $k\in K$.
While this inequality is satisfied by all integral solutions to the set $Q(z,v)$, it is violated by some of the solutions to its continuous relaxation.
We replace the inequalities \eqref{eq:left} in the formulation with  \eqref{eq:strongleft} to obtain a tighter formulation. 
We also replace inequalities \eqref{eq:right} in the formulation with the following valid inequality:
\begin{equation} \label{eq:strongright}
\sum_{b\in B: K^R(b) \ni k }c^i_b\le R(i,k)
\end{equation}
for all ${i\in I}$ and $k\in K$.

Note that for any $i\in I$, adding inequalities \eqref{eq:strongleft} and \eqref{eq:strongright} for the root node  implies that $\sum_{b\in B} c_b^i \le 1$ and therefore we can now drop inequality \eqref{eq:pickone} from the formulation.

\subsubsection{Breaking symmetry: Anchor features}\label{sec:anchoring}
If the variables of an integer program can  be  permuted without changing the structure of the problem, the integer program is called {\em symmetric}.
This poses a problem for MILP solvers (such as IBM ILOG CPLEX) since the search space increases exponentially, see Margot (2009).
The formulation \eqref{firstmodel} falls into this category  as there may be multiple alternate solutions that represent the same decision tree.
In particular, as we have discussed earlier in the paper, we consider a decision node that is not adjacent to leaf nodes and assume that the subtrees associated with  the left and right branches of this node are symmetric (i.e. they have the same topology). 
In this case, if the branching condition is reversed at this decision node (in the sense that the values of the $v$ variables associated with the chosen group are flipped), and, at the same time, the subtrees associated with the left and right branches of this node are switched, one obtains an alternate solution to the formulation corresponding to the same decision tree.
To avoid this, we designate one particular feature $j(g)\in J(g)$ of each group $g\in G$  to be the  \emph{anchor feature} of that group and enforce that if a group is selected for branching at such a node, samples with the anchor feature follow the left branch. More precisely, we add the following equations to the formulation:
\begin{equation}
\label{anchorfeatures}
z_{j(g)}^k = v^k_{g}
\end{equation}
for all  $g\in G$, and all $k\in K$ that is not adjacent to a leaf node and has symmetric subtrees hanging on the right and left branches.
While equations \eqref{anchorfeatures} lead to better computational performance, they do not exclude any decision trees from the feasible set of solutions.

\subsubsection{Deleting unnecessary variables}\label{sec:deletion}
Notice that the objective function \eqref{firstmodel:obj} uses variables $c_b^i$ only if it corresponds to a correct classification of the sample (i.e.,
${i\in I_{+}}$ and ${b\in B_{+}}$, or  ${i\in I_{-}}$ and $ {b\in B_{-}}$). 
Consequently,  the remaining $c_b^i$ variables can be projected out of the formulation 
without changing the value of the optimal solution. We therefore only define $c_b^i$ variables for
\begin{equation}\big\{(i,b): i\in I_{+}, b\in B_{+},\text{ or, } i\in I_{-}, b\in B_{-} \big\}\end{equation}
and write constraints \eqref{eq:left} and \eqref{eq:right}  for these  variables only. 
This reduces the number of $c$ variables and the associated constraints in the formulation by a factor of one half.
In addition, we delete equation \eqref{eq:pickone}.

Also note that the objective  function \eqref{firstmodel:obj} is maximizing a (weighted) sum of $c_b^i$ variables and the only constraints that restrict the values of these variables are inequalities  \eqref{eq:left} and  \eqref{eq:right} which have a right hand side of 0 or 1. 
Consequently, replacing the integrality constraints $c_b^i\in\{0,1\}$ with simple bound constraints  $1\ge c_b^i\ge0$, still yields optimal solutions that satisfy $c_b^i\in\{0,1\}$.
Hence, we do not require $c_b^i$ to be integral in the formulation and therefore significantly reduce the number of integer variables.

\subsubsection{Relaxing some binary variables}\label{sec:relaxation}
The computational difficulty of a MILP typically increases with the number of integer variables in the formulation 
and therefore it is desirable to impose integrality on as few variables as possible.
We next show that all of the $v$ variables and most of the $z$ variables take value $\{0,1\}$ in an optimal solution even when they are not explicitly constrained to be integral.

\begin{prop}\label{prop3}
	Every extreme point solution to \eqref{firstmodel} is integral even if  (i) variables $v^k_g$ are not declared integral for all $g\in G$ and decision nodes $k\in K$, and, (ii) variables $z^k_j$ are not declared integral for $j\in J$ and decision nodes $k\in K$ that are adjacent to a leaf node.
\end{prop}
\begin{proof}
	Assume  the claim does not hold and let  $\bar p = (\bar v,\bar z,\bar c)$ be an extreme point solution that is fractional.
	Let $K^L\subset K$ denote the decision nodes that are adjacent to leaf nodes and consider  node $a\not\in K^L$.
	First note that if  $\bar v^{a}_{b}$  is fractional, that is, if $1>\bar v^{a}_{b}>0$ for some  feature group $b\in G$,  then  $1>\bar v^a_g$ for all groups $g\in G$ as $\sum_{g\in G}\bar v^a_g=1$.
	Consequently, for this decision node we have  all $\bar z^a_j=0$ as  $\bar z^a_j\in\{0,1\}$ for $j\in J$. 
	This also implies that  $L(i,a)=0$ for all $i\in I$.
	In this case, for any $ g\in G$, the point $\bar p$ can be perturbed by setting the $ v^{a}_{ g}$ variable to 1 and setting the remaining $ v^{a}_{*}$ variables to 0 to obtain a point that satisfies the remaining constraints. 
	A convex combination of these perturbed points (with weights equal to $\bar v^{a}_{g}$ ) gives the point $\bar p$, a contradiction.
	Therefore all  $\bar v^k_g$ are integral for $g\in G$ and  $k\in K\setminus K^L$.
	
	Therefore, if $\bar p$ is fractional, then at least one of the following must hold:  either (i) $1>\bar v^{k}_{g}>0$ for some  $k\in K^L$ and $g\in G$, or, (ii) $1>\bar z^k_j>0$ for some $k\in K^L$ and $j\in J$, or, (iii) $1>c^i_b>0$ for some  $b\in B$ and $i\in I$. 
	As all these variables are associated with some decision node $k\in K^L$, we conclude that there exists a decision node $a\in K^L$ for which either $1>\bar v^{a}_{g}>0$ for some $g\in G$, or,  $1>\bar z^a_j>0$ for some $j\in J$, or,  $1>c^i_b>0$ for some $i\in I$ and $b\in \{b_+,b_-\}$ where $b_+\in B_+$ and $b_-\in B_-$ are the two leaf nodes attached to decision node $a$ on the left branch and on the right branch, respectively.
	
	Let $ I_a^+$ denote the set of samples in $I^+$ such that $\bar c^i_{b_+}>0$ and similarly,  let $ I_a^-$ denote the set of samples in $I^-$ such that $\bar c^i_{b_-}>0$. 
	If $\bar c^i_{b_+}\not = L(i,a)$, for some $i\in I_a^+$, then point $\bar p$ can be perturbed by increasing and decreasing $\bar c^i_{b_+}$ to obtain two new points that contain $\bar p$ in their convex hull, a contradiction. Note that $L(i,k) \in\{0,1\}$ for all $i\in I$ and $k\in K\setminus K^L$ and therefore these two points indeed satisfy all the constraints.
	Consequently, we conclude that  $\bar c^i_{b_+} = L(i,a)$ for all $i\in I_a^+$.  
	Similarly,  $\bar c^i_{b_-} = 1-L(i,a)$ for all $i\in I_a^-$. 
	Notice that this observation also implies that, if  $\bar c^i_{b_+}$ is fractional for some $i\in I_a^+$ or  $\bar c^i_{b_-}$ is fractional for some $i\in I_a^-$, then $L(i,a)$ is also fractional, which in turn implies that for some feature $h\in J$ we must have $\bar z^a_h>0$ fractional as well.

	Now assume there exists a feature  $h\in J$ such that $v^a_{g(h)}>\bar z^a_h>0$.  
	In this case increasing and decreasing $\bar z^a_h$ by a small amount and simultaneously updating the values of $\bar c^i_{b_+}$ for $i\in I_a^+$ and $\bar c^i_{b_-}$ for $i\in I_a^-$ to satisfy  $\bar c^i_{b_+} =  L(i,a)$ and  $\bar c^i_{b_-} = 1-L(i,a)$ after the update, leads to two new points that contain $\bar p$ in their convex hull. 
	Therefore, we conclude that $\bar z^a_h$ is either zero, or $\bar z^a_h = \bar v^a_{g(h)}$.
	
	So far, we have established that if $\bar c^i_b$ is fractional for some $i\in I_a^-\cup I_a^+$ and $b\in \{b_+,b_-\}$, then there is a fractional $\bar z^a_j$ variable for some  feature  $j\in J$. 
	In addition, we observed that if there is a fractional $\bar z^a_j$ for some $j\in J$ then there is a fractional  $\bar v^a_g$ for some $g\in G$. 
	Therefore, if $\bar p$ is not integral, there exists a feature group $d\in G$ such that $1>\bar v^a_d>0$. 
	As $\sum_{g\in G}\bar v^a_g=1$, this implies that there also exists a different group $e\in G\setminus\{d\}$ such that $1>\bar v^a_e>0$. 
	
	We can now construct two new points that contain $\bar p$ in their convex hull as follows:
	For the first point we increase $\bar v^a_d$ and decrease $\bar v^a_e$ by a small amount and for the second point we do the opposite perturbation. In addition, for both points we first update the values of $\bar z^a_j$ for all $j\in J$ with $g(j)\in\{b,d\}$ and $\bar z^a_j>0$ so that  $\bar z^a_h = \bar v^a_{g(h)}$ still holds.
	Finally, we perturb the associated $\bar c^i_b$ variables for $i\in I_a^-\cup I_a^+$ and $b\in \{b_+,b_-\}$ so that  $\bar c^i_{b_+} = L(i,a)$, for $i\in I_a^+$, and   $\bar c^i_{b_-} = 1-L(i,a)$ for all $i\in I_a^-$. 
	Both points are feasible and therefore we can conclude that $\bar p$ is not an extreme points, which is a contradiction. Hence $\bar p$ cannot be fractional.
	\qed
\end{proof}
We have therefore established that the only variables that need to be declared integral in the formulation \eqref{firstmodel} are the feature selection variables  $z^k_j$  for all features $j\in J$ and decision nodes $k\in K$ that are not adjacent to a leaf node.
Thus we have a formulation for training optimal decision trees, where the number of integer variables is {\em independent} of the number of samples.

\subsection{Handling numerical features}\label{sec:numericalfeatures}
To handle numerical features, we simply turn them into categorical features by binning them into intervals using deciles as thresholds.
Consequently, each numerical feature becomes a categorical feature with (up to) 10 possible values, depending on the decile it belongs to.
Therefore, one can use the model describe above without any further changes. 
However, this might lead to decision trees that branch on, for example, whether or not a numerical feature belongs to the 2nd or 7th quantiles, which of course is not a very interpretable condition. 
It is therefore desirable to branch on these features in a way that captures their ordinal nature.
To this end, we add additional constraints for these features to ensure that the branching decisions correspond to ``less than or equal to" or  ``greater  than or equal to" conditions.

More precisely, for each node $k\in K$ and for each group $g\in G$ that corresponds to a numerical feature, we create an additional variable $w_g^k$ to denote if the branching condition is of ``greater  than or equal to"  or  ``less than or equal to" form. We then require the associated $z^k_j$ variables for $j\in J(g)$ to take either increasing  (when  $w_g^k=1$)  or decreasing values (when  $w_g^k=0$). The additional constraints are,
\begin{equation*}\begin{alignedat}{4}
z_j^k &\geq z_{j+1}^k - w^k_{g} \quad&~ \forall j, j+1\in J(g) \\[.1cm] 
z_{j}^k &\geq z_{j-1}^k -  (1-w^k_{g}) \quad&~ \forall j, j-1\in J(g)\\[.1cm] 
w^k_{g} &\in\{0,1\}.&
\end{alignedat}
\end{equation*}

We note that it is possible to enforce ``less than or equal to" or  ``greater  than or equal to" form without using the additional variables $w$, by binarizing numerical features differently. However in this case the LP
formulation becomes  more dense and overall solution times are significantly slower. 

We  also note that an alternate way to break symmetry in this case is to set all $w^k_{g}$ variables  to 1 for $g\in G$  (without loss of generality) whenever  $k\in K$  is not adjacent to a leaf node and has symmetric subtrees hanging on the right and left branches. For balanced trees this property holds for all non-leaf nodes.
Clearly, if this symmetry breaking rule is used, one should not use anchor features described in Section \ref{sec:anchoring}.

\subsection{Controlling overfitting due to  combinatorial branching}\label{sec:limitcombinatorial}
As mentioned earlier, combinatorial branching may lead to overfitting when $|J(g)|$ is large for a categorical feature $g\in G$ as there are $2^{|J(g)|}$ possible ways to branch using this feature. 	To avoid overfitting,  we require the size of the subset used for branching to be either at most $max.card$ or at least  $(|J(g)|-max.card)$ for some input parameter  $max.card$.
To this end, for each node $k\in K$ and for each group $g\in G$ that corresponds to a categorical feature with $|J(g)|>max.card$, we create an additional variable $w_g^k$ and include the following constraints in the formulation,
\begin{equation*}\begin{alignedat}{4}
\sum_{j\in J(g)} z_j^k &\leq max.card + (|J(g)|-max.card ) (1-w_g^k )\\[.1cm] 
\sum_{j\in J(g)} z_j^k &\geq (|J(g)|-max.card ) - (|J(g)|-max.card ) w_g^k \\[.1cm] 
w^k_{g} &\in\{0,1\}.&
\end{alignedat}
\end{equation*}
We  again note that an alternate way to break symmetry in this case is to set all $w^k_{g}$ variables  to 1 for $g\in G$  whenever  $k\in K$  is not adjacent to a leaf node and has symmetric subtrees hanging on the right and left branches.

\subsection{Maximizing sensitivity/specificity}\label{sec:sensspec}
In many practical applications, especially those involving imbalanced datasets, the user's goal  is to  maximize
sensitivity (the true positive rate, or TPR),  while guaranteeing a certain level of specificity (the true negative rate, or TNR), or vice versa, instead of optimizing the total accuracy.  
While such problems cannot be addressed with heuristics such as CART (except by a trial-and-error approach to reweighting samples), our model \eqref{firstmodel}
readily lends itself to such a modified task. 
For example, if we intend to train a classifier with a guaranteed specificity (on the training set) of $0.95$, then we simply add  the following constraint to  \eqref{firstmodel}  
\begin{equation}\label{minspec}
\sum_{i\in I_{-}}\sum_{b\in B_{-}} c_b^i \geq  \lceil (1-0.95)|I_{-}| \rceil
\end{equation}
and change the objective function \eqref{firstmodel:obj} to 
\begin{equation}\label{maxrecall}
\sum_{i\in I_{+}}\displaystyle\sum_{b\in B_{+}} c_b^i.
\end{equation}

Likewise, we can produce a model that maximizes specificity while guaranteeing a certain level of sensitivity by switching the expressions in the constraint \eqref{minspec} and objective \eqref{maxrecall}.

\section{Computational Results}\label{sec:comp}

We now turn to computational experiments for which we used a collection of 10 binary (two-class) classification datasets. We obtained two of these datasets ({\em a1a} and {\em breast-cancer-wisconsin})  from LIBSVM \cite{LIBSVM}, one from FICO Explainable Machine Learning Challenge \cite{FICO} and the remaining  7  from the UCI Machine Learning repository \cite{UCI}.
These datasets were selected because they  fit into our framework  as majority their variables are either binary or categorical. 
Each dataset was preprocessed to have the binary form assumed by the formulation, 
with identified groups of binary variables. A summary description of the problems is given in Table \ref{data-describe}.

\begin{table}[h]
\centering
\caption{Summary description of the datasets}
\label{data-describe}
\begin{tabular}{rrrrrr}
\textbf{dataset}             & \textbf{\# Samples} & \textbf{\% Positive} & \textbf{\# Features} & \textbf{\# Groups} \\
a1a                          			& 1605  & 25\%   & 122   & 14  \\
breast-cancer-wisconsin 	(bc) 		& 695   & 65\%   & 90    & 9   \\
chess-endgame  				(krkp)		& 3196  & 52\%   & 73    & 36  \\
mushrooms   				(mush)      & 8124  & 52\%   & 111   & 20  \\
tic-tac-toe-endgame   		(ttt) 		& 958   & 65\%   & 27    & 9   \\
monks-problems-1 			(monks-1)   & 432   & 50\%   & 17    & 6   \\
congressional-voting-records (votes)	& 435   & 61\%   & 48    & 16  \\
spect-heart   				(heart)		& 267   & 79\%   & 44    & 22  \\
student-alcohol-consumption (student) 	& 395   & 67\%   & 109   & 31  \\
FICO Explainable ML Challenge (heloc)             & 9871  & 48\%   & 253   & 23
\end{tabular}
\end{table}
Each dataset/tree topology pair results in a MILP instance, which we implemented in 
Python 2.7 and then solved with IBM ILOG CPLEX 12.6.1 on a computational cluster, giving each instance access to 8 cores of an AMD Opteron 2.0 GHz processor.
Throughout this section, we will refer to our method as ODT (Optimal Decision Trees).

\subsection{Tuning the IP model}\label{sec:comp:IPtune}

We begin with some computational tests to illustrate the benefit  of various improvements to  the IP model that were discussed in $\S$\ref{sec:tricks}.
We only show results for five of the datasets: {\em a1a}, {\em bc}, {\em krkp}, {\em mush} and {\em ttt}, since for the other datasets, aside from {\em heloc},
the  IP is solved quickly and the effect of improvements is less notable, while for {\em heloc} the time limit was reached in all cases. 

We note that the deletion of unnecessary variables discussed in $\S$\ref{sec:deletion} seems to be
performed automatically by CPLEX in preprocessing, and so we do not report results relevant to this modeling choice.
However, we experiment with anchoring ($\S$\ref{sec:anchoring}), 
relaxing appropriate $z$ variables and $c$ variables ($\S$\ref{sec:relaxation}),
and strengthening the model ($\S$\ref{sec:strengthening}). In particular, we compare the model where none of the above techniques are 
applied (Nothing),  only relaxation and strengthening are applied (No Anchor), only anchoring and strengthening are applied (No Relax), only 
anchoring and  relaxation are applied (No Strength) and finally when all of the techniques are  applied (All).

In Table \ref{IPStrengtheningD3S200} we show the results for symmetric DTs of depths 3, while using  reduced datasets of  
200 randomly subsampled data instances. In each column we list the total time in seconds it took Cplex to to close the optimality gap to below the 
default tolerance and the total number of LPs solved in the process. In the case when Cplex exceeded 3 hours, 
the solve is terminated and a "*" is reported instead of the time.

\begin{table}[h]
	\centering
	\footnotesize
	\caption{IP Strengthening for depth 3 with 200 samples  - each table entry represents $\#$ seconds/number of LPs solved}
	\label{IPStrengtheningD3S200}
	\begin{tabular}{lcccccc}
		\textbf{Dataset}             & \textbf{Nothing} & \textbf{	No Anchor} & \textbf{	No Relax} & \textbf{No Strength} & \textbf{All} 	\\
a1a & */2443792 & */2422165 & */5660954 & 2670/598733 & 3098/1157891\\ 
bc  & 2193/50075 & 405/118193 &139/52375 &188/18121&44/18660\\
krkp  & 5377 /2766623 & 392/95623 & 3726/2702709 & 1434/291221 &320/131274\\
mush	&	31/26	&	22/20	&	12/65	&	22/26		&	23/49~	\\
ttt	&	1837/1914999	&	346/169235	&	71/63109	&	175/28588		&	31/10737	\\
\end{tabular}
\end{table}

As we see from Table  \ref{IPStrengtheningD3S200}, the data set with 200 data points make the IP difficult to solve for some data sets, such as {\em a1a}
but is  easy to some others, such as {\em bc} and {\em mush}. Hence in Table \ref{IPStrengtheningD3Svar} we show results for various sizes of data, selected so that the corresponding IP is not trivial but is still solvable within three hours. 
\begin{table}[h]
	\centering	\footnotesize
\caption{IP Strengthening for depth 3 with varying samples  - each table entry represents $\#$ seconds/number of LPs solved}
\label{IPStrengtheningD3Svar}
\begin{tabular}{lccccccc}
		\textbf{Dataset} & \textbf{Samples} & \textbf{Nothing} & \textbf{no Anchor} & \textbf{	No Relax.} & \textbf{No Strength}  & \textbf{All} 	\\
a1a		&	100	&	7262/2555737& 2541 / 1584533 & 503/426853 & 1352 /840813 & 170/104504\\
bc		&	300	&	7766/1013135& 5445/981711& 223/64411 &386/32262 & 349/53194\\	
krkp & 400 & */559764 & 6984/847235 &7533/1289615 & 2936 /97214 &3693/719622\\
mush	&	500	&	151/37~~ 	&	~41/0		&	55/1109		&	182/215		&	38/7		\\	
ttt		&	300	&	1394/404553&	946/226864	&424/88755	&	253/29869	&	35/12154	\\
	\end{tabular}
\end{table}

We can conclude from Tables  \ref{IPStrengtheningD3S200} and \ref{IPStrengtheningD3Svar} that our proposed strategies provide significant improvement 
in terms of computational time. In some cases, turning off an option  may outperform using all options; for example, turning off variable strengthening
 improves computational time
for {\em a1a}  compared to the {\em All}  option in  Table  \ref{IPStrengtheningD3S200}. However, it gives worse results for {\em bc}, {\em krkp} and {\em ttt}, hence we conclude that using all proposed
improvements is the best overall strategy.

Next we show the dependence of computational time on the tree topology and the size of the data set. In Table \ref{IPDepthvssizekrkp} we report these results for the {\em krkp}, {\em a1a}, and {\em bc} data set. Here, by depth 2.5 we refer to the topology shown in the upper right corner of Figure~\ref{fig:trees}, and by imbalanced, we refer to the topology shown in the bottom of Figure~\ref{fig:trees}.
In these experiments we terminated the Cplex run after 2 hours and when this happens we report "*" in the tables instead of the time.

\begin{table}[h]
	\centering
	\footnotesize
	\caption{Solution times  (in seconds) for {\bf krkp, bc and a1a}. }
	\label{IPDepthvssizekrkp}
	\begin{tabular}{llrrrrrr}
		\textbf{Topology}  &\bf  Data set &\bf100 &\bf 200 & \bf300 &\bf 400 &\bf 500 &\bf600 	\\
		depth2 		& krkp & 3& 7& 10 & 11 & 16 & 24\\
		depth 2.5 	& krkp & 12& 30& 134 & 106 &119  &803 \\
		depth 3 	& krkp & 268& 663& 2810 & 2267 &4039  & 5837\\
		imbalanced 	& krkp & 359& 3543& * & * & * & *\\[.3cm]
depth2 		& bc & 1& 3& 5 & 6 & 11 &10 \\
depth2.5 	& bc & 4& 30& 56 &  98& 173 & 199\\
depth3 		& bc &7 &44 & 349 & 2644& 1124 & *\\
imbalanced 	& bc & 3& 349& 1842 & 4958& * & *\\[.3cm]
		depth2 &  a1a & 3 & 9 & 12 & 18& 23 & 28 \\
		depth2.5&  a1a & 38 & 642	 & 1415 & 225& 302 & 2331\\
		depth3 &  a1a & 170 & 3098 & *& * & * & *\\
		imbalanced &  a1a & 1004 & * & * & * & * & *
	\end{tabular}
\end{table}

As one would expect, Table \ref{IPDepthvssizekrkp} shows that solving the IP to optimality becomes increasingly more difficult when the sample size increases and when  the tree topology becomes more complicated. 
However, the increase in solution time as sample size increases differs significantly among different datasets for the same tree topology depending on the number of features and groups of the dataset as well as how well the data can be classified using a decision tree.  
Note that even though the imbalanced trees and depth 3 trees have the same number of nodes, solving the IP for imbalanced trees is more challenging.
We believe this might be due to the fact that symmetry breaking using anchor features has to be disabled at the root node of  imbalanced trees, as the tree is not symmetric.

Restricting the number of features in the data can significantly reduce computational time. To demonstrate this, we run the following experiments: we first repeatedly apply the CART algorithm to  each data set, using 90\% of the data and default setting and thus not applying any particular restriction of the size of the tree. We then select groups that have been used for branching decision at least once in the CART tree. We then remove all other feature groups from the IP formulation (by setting the corresponding $v$ variables to $0$) and apply our ODT model to the reduced problem. The results in terms of time are listed in Table \ref{IPDepthvssizekrkpWS}. We can see that in many cases significant improvement in terms of time is achieved over results reported in Table \ref{IPDepthvssizekrkp}. We will discuss the effect of the feature selection on the prediction accuracy later in Section \ref{sec:comp:topologyA}. 

\begin{table}[h]
	\centering
	\footnotesize
	\caption{Solution times  (in seconds) for {\bf krkp, bc and a1a} using feature selection }
	\label{IPDepthvssizekrkpWS}
	\begin{tabular}{llrrrrrr}
		\textbf{Topology}  &\bf  Data set &\bf100 &\bf 200 & \bf300 &\bf 400 &\bf 500 &\bf600 	\\
		depth2 		& krkp &0 & 0& 1 & 1 & 2 & 2\\
		depth 2.5 	& krkp & 1& 2& 3 & 5 & 9 & 9\\
		depth 3 	& krkp & 1&3 &  7&  12& 19 & 20\\
		imbalanced 	& krkp &4 & 10& 20 &  30& 52 & 69\\[.3cm]
depth2 		& bc & 0& 1&  1&  1&  2& 4\\
depth2.5 	& bc & 1& 6& 17 & 22 & 48 & 59\\
depth3 		& bc &1 & 4& 6 & 11 & 22 &32 \\
imbalanced 	& bc &2 & 5&  11&  19& 47 & 54\\[.3cm]
		depth2 &  a1a & 2& 5& 10 & 12 & 15 & 18\\
		depth2.5&  a1a & 55&55 & 138 & 157 & 1011 & 824\\
		depth3 &  a1a &174 &1150& 4453 & 6278 & * &* \\
		imbalanced & a1a & 680&2731 & 6145 &  *& * & *
	\end{tabular}
\end{table}

\subsection{Effect of combinatorial branching } 

We next make a comparison to see the effect of the constraint on combinatorial branching  for categorical data which is discussed in Section \ref{sec:limitcombinatorial}.
When using this constraint with $max.card=1$ we recover "simple" branching rules where branching is performed using only one possible value of the 
feature, as is done in  \cite{bertsimas_dunn}. We compare simple branching denoted as {\em simple}, constrained branching using $max.card=2$, denoted by  {\em comb-con} and unconstrained branching, denoted as  {\em comb-unc}.
We only show the results for two data sets, {\em a1a} and {\em mush} because for the other data sets combinatorial branching did not produce different results as
most of the categorical features had only 2 or 3 possible values.
We compare decision trees  of  depths 2 and 3 trained using data sets of size 600.
Results averaged over 5 runs are shown in Table~\ref{oct-vs-odt}.

\begin{table}[h]
	\centering
	\caption{The average training (testing) accuracy  for combinatorial vs. simple branching using depth 2 and depth 3 trees}
	\label{oct-vs-odt}
	\begin{tabular}{lccccccc}
		&\multicolumn{3}{c}{\bf Depth 2}	&	\multicolumn{3}{c}{\bf Depth 3}\\
		\bf Dataset &\bf simple    &\bf comb-con &\bf comb-unc&\bf simple       &\bf comb-con &\bf comb-unc   \\
		a1a 	 & 82.2 (80.8) & 82.9 (81.0) & 83.3 (79.9)& 84.0 (80.8)& 84.8 (80.8)& 85.7 (80.1)\\
		mush & 95.8 (95.7) &99.6 (99.4) & 99.6 (99.4)&98.4 (97.7) &99.9 (99.4) &99.9 (99.3)\\

	\end{tabular}
\end{table}

 We see that for {\em mush} using combinatorial branching makes a significant improvement. 
In particular, for depth 3 trees it achieves a 99.4\% out-of-sample accuracy compared to 97.7\% for simple branching. 
We show the optimal  depth 3 tree for {\em mush} dataset in Figure \ref{fig:mushroom}.
However, for {\em a1a}   - even though unconstrained combinatorial branching  achieves good training accuracy they  do not generalize as well as simple branching rules.
In particular, the {\em a1a} dataset contains one group  (occupation) with many different possible values. 
Branching on this group results in combinatorially many possible decisions which leads to overfitting.  Adding a constraint with $max.card=2$ remedies the situation, while still providing a small improvement over simple branching. 

\begin{figure} [hbt] \caption{Optimal depth 3 decision tree for the Mushroom dataset with \%99.4 out of sample accuracy.} \label{fig:mushroom} 
	~\\
	\begin{tikzpicture}
	[
	sibling distance=8cm,  
	level distance = 1.75cm,
	xx/.style = {shape=rectangle, rounded corners, fill=white, draw, draw=red, align=center, very thick},
	eat/.style ={shape=rectangle,draw, fill=white, align=center},
	nay/.style ={shape=rectangle,draw, fill=gray!50, align=center} 
	]
	\node[xx] {Bruises?}
	child { 
		[sibling distance=4cm] node[xx] {Gill size} 
		child { [sibling distance=2cm] node[xx] {Is spore color\\\{chocolate, red\}?} 
			child {node[eat] {Eat} edge from parent node [  left] {No} } 
			child {node[nay] {Don't} edge from parent node [  right] {Yes} } 
			edge from parent node [above  left] {\{Broad\}} 
		}	
		child { [sibling distance=2cm] node[xx] {Is habitat \\\{woods, paths, waste\}?} 
			child {node[eat] {Eat} edge from parent node [  left] {Yes} } 
			child {node[nay] {Don't} edge from parent node [  right] {No} } 
			edge from parent node [above  right] {\{Narrow\}} 
		}
		edge from parent node [ above left] {Yes} 
	}
	child { 
		[sibling distance=4cm] node[xx] {Texture of stalk \\surface below ring(s)}
		child { [sibling distance=2cm]  node[xx] {\# of rings} 
			child {node[eat] {Eat} edge from parent node [  left] {$\{1,2\}$}} 
			child {node[nay] {Don't} edge from parent node [  right] {\{0\}} } 
			edge from parent node [above left] {\{Fibrous, smooth\}}  }
		child { [sibling distance=2cm] node[xx] {Is  odor \\\{almond, rusty, none\}?} 
			child {node[eat] {Eat} edge from parent node [  left] {Yes} }
			child {node[nay] {Don't} edge from parent node [  right] {No} } 
			edge from parent node [ above right] {\{Silky, scaly\}  } }
		edge from parent node [above right] {No} 
	};
	\end{tikzpicture}
\end{figure}
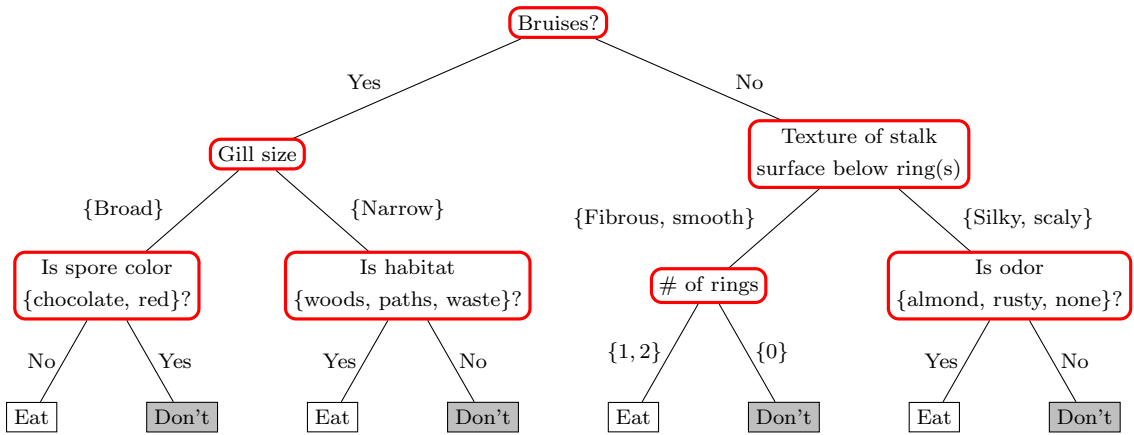

\subsection{Effect of constraints for numerical features.}
 Here we compare the effect of special constraints introduced for the  numerical features  in Section \ref{sec:numericalfeatures}. The results of this comparison are shown 
in Table \ref{numerical-compare}. When the constraint is imposed, the feature group is treated as numerical, and this formulation is label with "n", for numerical. When the constraint is not imposed, then the group is treated as if the original feature is categorical, and the formulation is labeled with "c", for categorical. We compare both accuracy and time averaged from 5 runs with 30 mins limit. 

%
%

\begin{table}[h]
\centering
\caption{The average training (testing) accuracy/solution time with or without constraints for numerical features.}
\label{numerical-compare}
\begin{tabular}{llllll}
\textbf{Dataset} & n/c & \textbf{Depth 2} & \textbf{Depth 2.5} & \textbf{Depth 3} & \textbf{Imbalanced} \\
a1a		& n & 82.9 (81.0)/22 & 84.7 (80.5)/748 & 84.8 (80.8)/1800 & 84.7 (80.7)/1800  \\
			&c & 82.9 (81.0)/24 & 84.5 (81.0)/1191 & 84.7 (80.3)/1800 & 84.8 (80.1)/1800\\[.1cm]
bc		&n & 96.7 (96.6)/6 & 97.5 (95.4)/70 & 97.8 (96.2)/608 & 97.8 (95.6)/1749 \\
			&c & 96.7 (96.0)/6 & 97.8 (94.9)/272 &98.4 (94.7)/1800 & 98.5 (95.5)/1800\\[.1cm]
heloc	&n & 72.0 (71.2)/7 & 73.3 (70.6)/119 & 73.8 (70.0)/515 & 74.1 (69.8)/1788  \\
			&c & 72.9 (70.4)/13.6& 74.9 (68.0)/1711 & 75.9 (68.1)/1800 & 75.6 (68.6)/1800\\[.1cm]
student	&n & 92.1 (90.5)/1 & 92.6 (91.0)/8 & 92.6 (90.5)/25 & 93.1 (91.5)/127\\
			&c & 92.1 (90.5)/1 & 92.6 (91.0)/10 &92.6 (90.5)/62 & 93.1 (91.0)/113
\end{tabular}
\end{table}

We observe that overall adding the special constraint to impose the numerical nature of the group improves the testing accuracy and saves computational time. 

\subsection{Comparison with CART depth 3 trees }\label{sec:comp:topologyA}

We next focus on comparing the accuracy of ODTs with CART. 
We consider 4  different tree topologies  for ODTs:  depth 2, depth 2.5, depth 3 and  imbalanced.
We use CART as implemented in the package rpart for R \cite{rpart}.
We compare the performance of ODT to CART by restricting the maximum depth of the learned CART trees to 3, thus allowing at most 8 leaf nodes, which is the maximum that our trees can have.
We note that this does not mean the learned CARTs have the same topology as our ODTs.

We also investigate the effect of {\em feature selection} by running CART first and  considering only the features used by CART in constructing ODTs.
For each dataset, we generate 5 random training/testing splits of the dataset by sampling without replacement and report the averages.
We use $90\%$ of the data for training CART and we use $\min \{90\%,600\}$ data points for training ODTs. 



In Tables \ref{Cart-compare} and  \ref{Cart-compare3} we show the results for ODTs trained for up to 30 minutes with and without   feature selection, respectively, and compare with CART trees of depth 3. 
In both tables we list the average training and testing accuracy, in percentages, over the 5 runs.  We highlight in bold the best testing accuracy 
achieved by the ODTs if it is more than $1\%$ larger than that achieved by CART, and reversely, highlight accuracy of CART when it is more than
$1\%$ larger than best accuracy of ODT. The standard deviation in all cases is fairly small,  typically around $0.2-0.3\%$.

\begin{table}[h]
\centering
\caption{The average training (testing) accuracy   with 30 mins limit without feature selection.}
\label{Cart-compare}
\begin{tabular}{llllll}
\textbf{Dataset}             & \textbf{Depth 2} & \textbf{Depth 2.5} & \textbf{Depth 3} & \textbf{Imbalanced} & \textbf{CART-D3}\\
a1a & 82.9 (80.9) & 84.7 (80.5) & 84.7 (\textbf{81.0}) & 85.2 (80.0) & 82.0 (79.3) \\
bc & 96.7 (\textbf{96.6}) & 97.5 (95.6) & 97.8 (94.9) & 97.9 (96.4) & 96.0 (94.6)\\
heloc & 72.4 (69.8) & 73.2 (70.0) & 73.7 (69.6) & 73.0 (68.8)  & 70.8 (71.0)  \\
krkp & 86.7 (87.0) & 93.2 (93.9) & 93.3 (93.9) & 94.1 (\textbf{94.1}) &90.4 (90.3) \\
mush & 99.6 (99.4) & 100.0 (99.5) & 100.0 (99.7) & 100.0 (99.6) & 99.4 (99.3) \\
ttt & 71.8 (67.7) & 77.0 (72.7) & 79.3 (74.2) & 81.9 (\textbf{79.5}) & 75.3 (73.1)\\
monks-1 & 78.2 (74.1) & 84.1 (76.8) & 89.6 (82.3) & 100.0 (\textbf{100.0}) & 76.6 (76.8) \\
votes & 96.2 (95.5) & 96.9 (93.6) & 97.4 (94.1) & 98.0 (95.0) & 95.7 (95.9) \\
heart & 85.5 (88.1) & 88.6 (89.6) & 88.7 (89.6) & 90.7 (85.2) &88.5 (\textbf{91.1}) \\
student & 92.8 (\textbf{91.0}) & 93.1 (91.0) & 93.3 (89.0) & 93.4 (89.5) & 89.5 (86.0)
\end{tabular}
\end{table}


\begin{table}[h]
\centering
\caption{The average training (testing) accuracy  with 30 mins limit with feature selection} 
\label{Cart-compare3}
\begin{tabular}{llllll}
\textbf{Dataset}             & \textbf{Depth 2} & \textbf{Depth 2.5} & \textbf{Depth 3} & \textbf{Imbalanced} & \textbf{CART-D3}\\
a1a & 82.9 (\textbf{81.0}) & 84.7 (80.5) & 84.8 (80.8) & 84.7 (80.7) & 82.0 (79.3) \\
bc & 96.7 (\textbf{96.6}) & 97.5 (95.4) & 97.8 (96.2) & 97.8 (95.6) & 96.0 (94.6) \\
heloc & 72.0 (71.2) & 73.3 (70.6) & 73.8 (70.0) & 74.1 (69.8) & 70.8 (71.0) \\
krkp & 86.7 (87.0) & 93.2 (93.9) & 93.2 (\textbf{93.9}) & 94.6 (93.8) & 90.4 (90.3) \\
mush & 99.6 (99.4) & 99.9 (99.4) & 99.9 (99.4) & 100.0 (99.6) & 99.4 (99.3) \\
ttt & 71.8 (67.7) & 77.0 (72.7) & 79.3 (74.2) & 81.9 (\textbf{79.5}) & 75.3 (73.1) \\
monks-1 & 78.2 (74.1) & 84.1 (76.8) & 89.6 (82.3) & 100.0 (\textbf{100.0}) & 76.6 (76.8) \\
votes & 95.9 (95.5) & 96.3 (95.0) & 96.7 (95.0) & 97.3 (96.8) &  95.7 (95.9) \\
heart & 85.5 (88.1) & 88.6 (90.4) & 88.6 (90.4) & 90.2 (88.9) & 88.5 (91.1) \\
student & 92.1 (90.5) & 92.6 (91.0) & 92.6 (90.5) & 93.1 (\textbf{91.5}) & 89.5 (86.0)
\end{tabular}
\end{table}
In Table \ref{Cart-compare} we see that testing accuracy achieved by ODTs after 30 minutes of training is significant better than that of depth 3 CART. Comparing Tables \ref{Cart-compare} and  \ref{Cart-compare3}, we see that on average the feature selection typically degrades training accuracy but results in better testing accuracy. This can be explained by the fact that reducing the number of features prevents the ODTs from overfitting. 
This observation suggest that using feature selection, especially  for larger trees could be beneficial not only for computational speedup but for better accuracy.

 We next repeat the same experiments from Tables \ref{Cart-compare} and \ref{Cart-compare3} with a 5 minute time limit  on Cplex and report the results in Tables  \ref{Cart-compare4} and \ref{Cart-compare6}. Note that the time for feature selection is negligible. 
 
\begin{table}[h]
\centering
\caption{The average training (testing) accuracy   with 5 mins limit  without feature selection.}
\label{Cart-compare4}
\begin{tabular}{llllll}
\textbf{Dataset}             & \textbf{Depth 2} & \textbf{Depth 2.5} & \textbf{Depth 3} & \textbf{Imbalanced} & \textbf{CART-D3}\\
a1a & 82.9 (\textbf{80.9}) & 84.5 (80.6) & 84.4 (80.9) & 83.5 (80.4) & 82.0 (79.3) \\
bc & 96.7 (\textbf{96.6}) & 97.5 (95.6) & 97.7 (96.4) & 97.6 (96.2) & 96.0 (94.6) \\
heloc &  72.4 (69.8) & 72.0 (69.1) & 66.2 (65.0) & 58.2 (57.6)  & 70.8 (71.0)  \\
krkp & 86.7 (87.0) & 93.2 (\textbf{93.9}) & 92.1 (92.1) & 92.9 (92.9) & 90.4 (90.3) \\
mush & 99.6 (99.4) & 100.0 (99.5) & 100.0 (99.7) & 100.0 (99.7) & 99.4 (99.3) \\
ttt & 71.8 (67.7) & 77.0 (72.7) & 78.7 (74.0) & 77.5 (\textbf{75.0}) &75.3 (73.1) \\
monks-1 & 78.2 (74.1) & 84.1 (76.8) & 89.6 (82.3) & 100.0 (\textbf{100.0}) & 76.6 (76.8) \\
votes & 96.2 (95.5) & 96.9 (93.6) & 97.3 (94.5) & 97.5 (92.7) & 95.7 (95.9) \\
heart & 85.5 (88.1) & 88.6 (89.6) & 88.7 (89.6) & 90.4 (88.1) & 88.5 (\textbf{91.1}) \\
student & 92.8 (91.0) & 93.0 (\textbf{91.5}) & 93.0 (90.0) & 87.7 (86.5) & 89.5 (86.0)
\end{tabular}
\end{table}


\begin{table}[h]
\centering
\caption{The average training (testing) accuracy   with 5 mins limit with feature selection}
\label{Cart-compare6}
\begin{tabular}{llllll}
\textbf{Dataset}             & \textbf{Depth 2} & \textbf{Depth 2.5} & \textbf{Depth 3} & \textbf{Imbalanced} & \textbf{CART-D3}\\
a1a & 82.9 (\textbf{81.0}) & 84.6 (80.4) & 84.4 (80.2) & 84.6 (80.7) & 82.0 (79.3) \\
bc & 96.7 (\textbf{96.6}) & 97.5 (95.4) & 97.7 (95.8) & 97.7 (95.2) & 96.0 (94.6) \\
heloc & 72.0 (71.2) & 73.3 (70.6) & 73.7 (69.7) & 73.5 (70.9)  & 70.8 (71.0)  \\
krkp & 86.7 (87.0) & 93.2 (93.9) & 93.2 (\textbf{93.9}) & 94.6 (93.8) & 90.4 (90.3) \\
mush & 99.6 (99.4) & 99.9 (99.4) & 99.9 (99.4) & 99.9 (99.4) & 99.4 (99.3) \\
ttt & 71.8 (67.7) & 77.0 (72.7) & 78.7 (74.0) & 77.5 (\textbf{75.0}) & 75.3 (73.1) \\
monks-1 & 78.2 (74.1) & 84.1 (76.8) & 89.6 (82.3) & 100.0 (\textbf{100.0}) & 76.6 (76.8) \\
votes & 95.9 (95.5) & 96.3 (95.0) & 96.7 (95.0) & 97.3 (96.8) & 95.7 (95.9) \\
heart & 85.5 (88.1) & 88.6 (90.4) & 88.6 (90.4) & 90.2 (88.9) & 88.5 (91.1) \\
student & 92.1 (90.5) & 92.6 (91.0) & 92.6 (90.5) & 93.1 (\textbf{91.5}) & 89.5 (86.0)
\end{tabular}
\end{table}

Comparing Tables \ref{Cart-compare} and  \ref{Cart-compare4}, we do not see a significant difference in accuracy for depth 2 and depth 2.5 ODTs due to the reduction of the time limit from 30 minutes to 5 minutes. 
For depth 3 ODTs, and the imbalanced trees however, both training and testing performance gets noticeably worse due to the reduction of the time limit.
Comparing Tables \ref{Cart-compare4} and  \ref{Cart-compare6}, we see that in most cases feature selection helps in terms of both training and testing accuracy.

Overall the testing accuracy degrades between  \ref{Cart-compare} and  \ref{Cart-compare6}, but not very significantly, thus  we conclude that feature selection helps for larger trees independent of the time limit. 
Moreover,  average testing accuracy of ODTs obtained only after 5 minutes  of computation using feature selection seems to be  similar to testing accuracy with 30 minute time limit (with or without feature selection) and thus still outperforms CART.
We should also note that when the IPs are terminated earlier, the optimality gap is usually larger but it often happens that an optimal or a near optimal integral solution  is already obtained by CPLEX.

\subsection{Effect of training set size}

To demonstrate the effect of the training set size on the resulting testing accuracy we present the appropriate comparison in 
Table \ref{sample-compare-d3}. In these experiments we run Cplex with a 30 minute time.

\begin{table}[h]
\centering
\caption{Comparison of training  (testing) accuracy across training data sizes  with 30 minutes limit and feature selection}
\label{sample-compare-d3}
\begin{tabular}{lcrrrrrrr}
\textbf{Dataset}  & \textbf{Topology} & \textbf{600~}& \textbf{1200~} & \textbf{1800~} & \textbf{2400~} \\
a1a & 2 &82.9 (81.0)& 82.4 (79.3) & 82.0 (79.6) & 82.0 (79.6) \\
krkp & 2 &86.7 (87.0)& 86.8 (87.0) & 86.8 (87.1) & 86.8 (87.2) \\
mush & 2&99.6 (99.4) & 99.5 (99.4) & 99.4 (99.4) & 99.4 (99.4) \\
heloc & 2 &72.0 (71.2)&72.2 (70.8)& 71.9 (71.2) & 71.7 (71.2)\\[.2cm]

a1a & 2.5 &84.7 (80.5)& 83.7 (80.0) & 83.4 (79.6) & 83.4 (79.6) \\
krkp & 2.5 & 93.2 (93.9)& 93.8 (93.8) & 93.6 (94.0) & 93.7 (94.1) \\
mush & 2.5 &99.6 (99.4)& 99.8 (99.5) & 99.7 (99.6) & 99.7 (99.6) \\
heloc & 2.5 &73.3 (70.6)&73.1 (70.9)& 72.6 (70.6) & 72.3 (71.5) \\[.2cm]

a1a & 3 &84.7 (80.7)& 83.6 (79.6) & 83.3 (80.2) & 83.3 (80.2) \\
krkp & 3&94.6 (93.8)& 93.8 (93.8) & 93.6 (94.0) & 93.7 (94.1) \\
mush & 3&100.0 (99.6) & 99.9 (99.6) & 99.9 (99.7) & 99.8 (99.8) \\
heloc & 3 &73.8 (70.0)&73.5 (70.9)& 72.9 (71.3) & 72.5 (71.4) \\[.2cm]

a1a & IB &84.8 (80.8)& 83.6 (79.2) & 82.5 (79.6) & 82.2 (79.0) \\
krkp & IB &93.2 (93.9)& 94.5 (93.7) & 94.2 (93.9) & 94.1 (94.1) \\
mush & IB &99.9 (99.4)& 100.0 (99.8) & 100.0 (100.0) & 100.0 (100.0) \\
heloc & IB &74.1 (69.8)&73.2 (71.0)& 72.1 (71.4) & 72.0 (71.4)
\end{tabular}
\end{table}

We observe that in most cases increasing the size of the training data narrows the gap between training and testing accuracy. This can happen for two reasons -  because  optimization progress slows down and training accuracy drops and/or because there is less overfitting. 
For example, for {\em a1a} it appears to be harder to find the better tree and so both the training and the testing accuracy drops, while for {\em mush} testing accuracy gets better, as the gap between training and testing accuracy closes.  We also see, for example in the case of  {\em mush} and {\em krkp},  the 
effect of the increase of the data set tends to diminish as the gap between training and testing accuracy. 
This is a common behavior for machine learning models, as larger training data tends to be more representative with respect to the entire data set. However, in our case,
we utilize the larger data set to perform prior feature selection and as a result 
relatively small training sets are often sufficient for training of the ODTs. 
Hence, the  computational burden of solving IPs to train the ODTs is balanced by the lack of need to use large training sets. 

\subsection{Choosing the tree topology}\label{sec:comp:topology}

		
In this section we discuss how to chose the best tree topology via cross-validation and compare the accuracy obtained by  the chosen topology
to the accuracy of trees obtained by CART with cross-validation.

For each dataset we randomly selected $90\%$ of the data points to use for training and validation, leaving the remaining data for final testing.  For the smaller data sets, we select the best topology using standard $5$-fold cross validation. For large data sets such as  {\em a1a, bc, krkp, mush} and {\em ttt}, we instead repeat the following experiment 5 times: we randomly select $600$ data points as the training set and train a tree of each topology on this set. The remaining data is used as a validation set   and
 we  compute the accuracy of each trained tree on this set.   After $5$ experiments, we select the topology that has the best average validation accuracy. 
 We then retrain the tree with this topology  and report the testing accuracy using the hold-out $10\%$.   We train CART with $90\%$ of the data points, allowing it to choose the tree depth using its default setting and then report the testing accuracy using the hold-out set. 
 We summarize the results in Table \ref{cv-compare} where for each method we  list the average testing accuracy and the average  number of leaves in the tree chosen via cross-validation. We set ODT time limit to 30 mins and used feature selection from CART trained on $90\%$ of each dataset.

\begin{table}[h]
\centering
\caption{Comparison of  testing accuracy and size of cross validated trees vs. CART }
\label{cv-compare}
\begin{tabular}{lcccccc}
\bf Dataset & \bf ODT &\bf ave. \# of leave &\bf	CART 	 &\bf ave. \# of  leaves \\
a1a		& { \bf 80.9}   & \bf 6.8        	       & 79.6       & 9.6\\
bc 		&{ \bf 96.0}    & { 4.8}            & 94.9        & \bf 4.2\\
heloc   &\bf 71.4          &4.8                   & 71.0       & \bf 3.6\\
krkp    &93.6          & { \bf 6.8}             &{ \bf 96.6}   & 9\\
mush	&\bf 99.8          & { 7.6}    	       & 99.3	     &\bf 3\\
ttt		&81.0	       & { \bf 8.0}     & \bf 93.1 &20.2\\
monks-1	& \bf 100.0            & { \bf 8.0}    & { 82.3}	 &8.6	              	\\
votes	&{ \bf 95.7}	   & 7.2          & 95.5         &\bf 2.4                \\
heart   &  \bf 89.6        & {7.2}     &  {88.9}   & \bf 5         \\
student & \bf 90.5         &\bf 4.8           &  86.0           &6.2           
\end{tabular}
\end{table}

We can summarize the results in Table \ref{cv-compare} as follows: in most cases, either ODTs outperform CARTs  in terms of accuracy or else they tend to have a significantly simpler structure than the CART trees. In particular, for data sets {\em a1a}, {\em student} and {\em bc}   that contain interpretable human-relatable data, ODTs perform better in terms  accuracy and better or comparably in interpretability, undoubtedly because there exist simple shallow trees that make good predictors for such data sets, and the exact optimization method such as ours can find such trees, while a heuristic, such as CART may not. 
 On the other hand, on the dataset {\em ttt} (which describes various positions in a combinatorial game), simple $2$ or $3$ levels of decision are simply not enough to predict the game outcome. In this case, we see that CART can achieve better accuracy, but at the cost of using much deeper trees. 
A similar situation holds for  {\em krkp}, but to a lesser extent.  Finally, {\em monks-1} data set is an artificial data set, classifying robots using simple features describing parts of each robot. Classification
in {\em monks-1}  is based on simple rules that can be modeled using shallow trees and ODT performance is much better on that data set than that of CART. 
In conclusion, our results clearly demonstrate that when classification can be achieved by a small interpretable tree, ODT outperforms CART in accuracy and interpretability.

\subsection{Training depth-2 tree on full {\em heloc} data.}
We performed a more detailed study of the {\em heloc} data set which was introduced in the FICO interpretable machine learning competition  \cite{FICO}. 
The authors of  the winning approach
\cite{fico_oktay} produced a model for this data set which can be represented  as a depth-2 decision tree achieving $71.7$ testing accuracy. Here we show how we are able to obtain comparable results with our approach. First we applied feature selection using CART, making sure that at least 4 features are selected. Then we trained a depth-2 tree using our ODT model and $90\%$ of the data points ($8884$ points). The optimal solution was obtained within $405$ seconds and the resulting testing accuracy is $71.6$. The corresponding CART model gives $71.0$ testing accuracy. 

\subsection{Results of maximizing sensitivity/specificity}\label{sec:sensspec_results}
We now present computational results related to the maximization of sensitivity or specificity, as discussed in Section \ref{sec:sensspec}.
We will focus on the {\em bc} dataset, which contains various measurements of breast tumors. The positive examples in this data sets are the individuals with 
malignant tumors in the breast. Clearly, it is vitally important to correctly identify all (or almost all) 
positive examples, since missing a positive example may result in sending a individual who may need cancer treatment home without recommending further tests or treatment. On the other hand,  placing a healthy individual into the malignant group, while undesirable, is less damaging, since further tests will simply correct the error. Hence, the goal should be maximizing specificity, while 
constraining sensitivity. Of course, the constraint on the sensitivity is only guaranteed on the training set. 
In Table \ref{bcd3}  we present the results of solving such model using $\min(\lceil.9n\rceil, 600)$  samples and the resulting testing sensitivity (TPR) and specificity (TNR). 
We report average and variance over 30 runs. 


\begin{table}[h]
	\centering
	\caption{TPR vs. TNR, breast cancer data, depth 2 and depth 3  trees }
	\label{bcd3}
	\begin{tabular}{l|ll|c c l|ll|l}
		\multicolumn{4}{c}{\bf Depth 2}	&&	\multicolumn{4}{c}{\bf Depth 3}\\
		\multicolumn{2}{c}{\bf Training}	&	\multicolumn{2}{c}{\bf Testing}&&\multicolumn{2}{c}{\bf Training}	&	\multicolumn{2}{c}{\bf Testing}\\
		\bf TPR	& \bf TNR & \bf TPR&	\bf	TNR &~~~&\bf TPR	&\bf TNR &\bf 	TPR&\bf	TNR\\
		100	&79.6&99.1	&		76.8&&	100&	91.6&		97.2&		83.6\\
		99.5&85.4&	98.9&		82.4&&	99.5&	94.6&		97.4&		89.7\\
		99	&89.5&	97.7&		89.4&&	99	&	97.2&		96.8&		90.0\\
		98.5&92	 &	98.1&		90.9&&	98.5&	97.2&		97.2&		90.9\\
		98	&92.7&	97.7&		91.0&&	98	&	98.7&		96.4&		94.6\\
		97	&95.8&	97.5&		94.7&&	97	&	99.4&		96.6&		96.1\\
		96	&97.3&	96.4&		93.9&&	96	&	99.9&		94.2&		94.7\\
		95	&98.4&	96.2&		98.0&&	95	&	100.0&		93.9&		93.0\\
	\end{tabular}
\end{table}

We observe that, while depth-2 trees deliver worse specificity in training than depth-3 trees, they have better generalization and hence closely maintain the desired true positive rate.  This is also illustrated in Figure \ref{fig:SpecD1}.

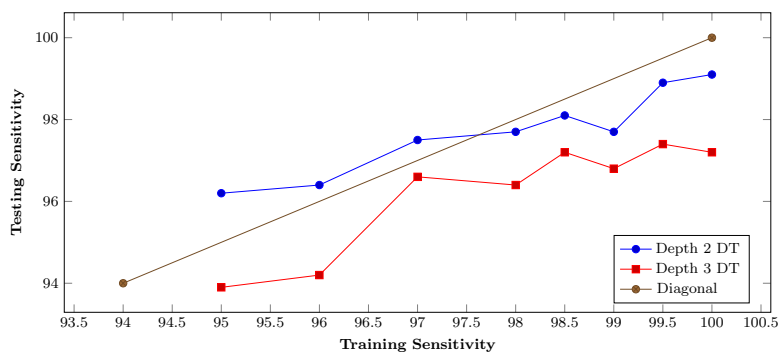
\begin{figure} [hbt] \caption{Breast Cancer Data, Training v.s. Testing Sensitivity} \label{fig:SpecD1}
	\vskip.5cm
	\begin{center}
		\begin{tikzpicture}[scale=0.7]
		\begin{axis}[
		xlabel={\bf Training Sensitivity},
		ylabel={\bf Testing Sensitivity},
		legend pos=south east,
		width=15cm,
		height=\axisdefaultheight]
		\addlegendentry{Depth 2 DT}
		\addplot coordinates {(100,	99.1)(99.5,98.9)(99,97.7)(98.5,	98.1)(98,97.7)(97,97.5)(96,96.4)(95,96.2)};
		\addlegendentry{Depth 3 DT}
		\addplot coordinates {(100,97.2)(99.5,97.4)(99,96.8)(98.5,97.2)(98,96.4)(97,96.6)(96,94.2)(95,93.9)	};
		\addlegendentry{Diagonal~~~~~}
		\addplot coordinates {(100,100)(94,94)	};
		\end{axis}
		\end{tikzpicture}
	\end{center}
\end{figure}

\section{Concluding remarks}

We have proposed an integer programming formulation for constructing optimal binary classification trees for data consisting of categorical features. 
This integer programming formulation takes problem structure into account and, as a result, the number of integer variables in the formulation is independent of the size of the training set. We show that the resulting MILP can be solved to optimality in the case of small decision trees; in the case of larger topologies, a good solution can be obtained within a set time limit. 
We show that our decision trees tend to outperform those produced by CART, in accuracy and/or interpretability. 
Moreover, our formulation can be extended to optimize specificity or sensitivity instead of accuracy, which CART cannot do. 

Our formulation is more specialized than that proposed recently in  \cite{bertsimas_dunn} and is hence is easier to solve by an MILP solver. However, our model 
allows flexible branching rules for categorical variables, 
as those allowed by CART.  In addition the formulations proposed in \cite{bertsimas_dunn} are not particularly aimed at interpretability.

Several extensions and improvements should be considered in future work. 
For example, while the number of integer variables does not depend on the size of the training set, the number of continuous variables and the problem difficulty increases with the training set size. 
Hence, we plan to consider various improvements to the solution technique which may considerably reduce this dependence.

\bibliographystyle{plain}

\begin{thebibliography}{10}
	
	\bibitem{bennett_blue1}
	K.P. Bennett and J.~Blue.
	\newblock Optimal decision trees.
	\newblock Technical Report 214, Rensselaer Polytechnic Institute Math Report,
	1996.
	
	\bibitem{bennett_blue2}
	K.P. Bennett and J.A. Blue.
	\newblock A support vector machine approach to decision trees.
	\newblock In {\em Neural Networks Proceedings of the IEEE World Congress on
		Computational Intelligence}, ~volume 3, pages 2396--2401, 1998.
	
	\bibitem{bertsimas_dunn}
	D.~Bertsimas and J.~Dunn.
	\newblock Optimal classification trees.
	\newblock {\em Machine Learning}, 106.7:1039--1082, 2017.
	
	\bibitem{bertsimas_shioda}
	D.~Bertsimas and R.~Shioda.
	\newblock Classification and regression via integer optimization.
	\newblock {\em Operations Research}, 55(2):252--271, 2017.
	
	\bibitem{Breimanbook}
	L.~Breiman, J.~H.~Friedman, R.~A.~Olshen, and C.~J.~Stone.
	\newblock { Classification and Regression Trees}.
	\newblock Chapman \& Hall, New York, 1984.
	
	\bibitem{BreimanRF}
	Leo Breiman.
	\newblock Random forests.
	\newblock {\em Machine Learning}, 45(1):5--32, 2001.
	
	\bibitem{fico_oktay} 
	S. Dash, O. G\"unl\"uk, D. Wei.
	\newblock { Boolean decision rules via column generation.}
	\newblock {\em Advances in Neural Information Processing Systems,}
	\newblock{ Montreal, Canada, December 2018. }\newcommand{\ks}[1]{{\color{blue}#1}}
	
	\bibitem{FICO}
	FICO Explainable Machine Learning Challenge 
	\newblock {https://community.fico.com/s/explainable-machine-learning-challenge} 
	
	\bibitem{LIBSVM}
	Chih-Chung Chang and Chih-Jen Lin.
	\newblock {LIBSVM}: A library for support vector machines.
	\newblock {\em ACM Transactions on Intelligent Systems and Technology},
	2:27:1--27:27, 2011.
	
	\bibitem{hyafil_rivest}
	L.~Hyafil and R.L. Rivest.
	\newblock Constructing optimal binary decision trees is np-complete.
	\newblock {\em Information Processing Letters}, 5(1):15--17, 1976.
	
	\bibitem{Kotsiantis2013}
	S.~B. Kotsiantis.
	\newblock Decision trees: a recent overview.
	\newblock {\em Artificial Intelligence Review}, 39(4):261--283, 2013.
	
	\bibitem{UCI}
	M.~Lichman.
	\newblock {UCI} machine learning repository, 2013.
	
	\bibitem{malioutov_varshney}
	D.M. Malioutov and K.R. Varshney.
	\newblock Exact rule learning via boolean compressed sensing.
	\newblock In {\em Proceedings of the 30th International Conference on Machine
		Learning}, volume~3, pages 765--773, 2013.
	
	\bibitem{murthy_salzberg}
	Sreerama Murthy and Steven Salzberg.
	\newblock Lookahead and pathology in decision tree induction.
	\newblock In {\em Proceedings of the 14th International Joint Conference on
		Artificial Intelligence}, volume 2, pages 1025--1031, San
	Francisco, CA, USA, 1995. Morgan Kaufmann Publishers Inc.
	
	\bibitem{norouzi_et_al}
	M.~Norouzi, M.~Collins, M.A. Johnson, D.J. Fleet, and P.~Kohli.
	\newblock Efficient non-greedy optimization of decision trees.
	\newblock In {\em Advances in Neural Information Processing Systems}, pages
	1720--1728, 2015.
	
	\bibitem{Quinlan}
	J.~Ross Quinlan.
	\newblock {\em C4.5: Programs for Machine Learning}.
	\newblock Morgan Kaufmann Publishers Inc., San Francisco, CA, USA, 1993.
	
	\bibitem{rpart}
	Terry Therneau, Beth Atkinson, and Brian Ripley.
	\newblock rpart: Recursive partitioning and regression trees.
	\newblock Technical report, 2017.
	\newblock R package version 4.1-11.
	
	\bibitem{wang_rudin}
	T.~Wang and C.~Rudin.
	\newblock Learning optimized or's of and's.
	\newblock Technical report, 2015.
	\newblock arxiv:1511.02210.
	
\end{thebibliography}

\end{document}